\documentclass [11pt]{article} 
\usepackage{fullpage,wrapfig}
\usepackage{amsmath, amsthm, amssymb, algorithm,algpseudocode,enumerate}
\usepackage[justification=centering]{caption}
\usepackage{bm}
\usepackage{color}
\usepackage{algpseudocode}
\usepackage[usenames,dvipsnames,svgnames,table]{xcolor}
\definecolor{darkgreen}{rgb}{0.0,0,0.9}
\usepackage[colorlinks=true,pdfpagemode=UseNone,citecolor=OliveGreen,linkcolor=BrickRed,urlcolor=BrickRed,
pagebackref]{hyperref}
\usepackage{thmtools,thm-restate}
\usepackage{verbatim}
\usepackage{enumitem}
\usepackage{nicefrac}

\declaretheorem[name=Theorem,numberwithin=section]{theorem}
\declaretheorem[sibling=theorem]{lemma}

\declaretheorem[sibling=theorem]{claim}

\declaretheorem[sibling=theorem]{corollary}

\declaretheorem[sibling=theorem]{remark}

\declaretheorem[sibling=theorem]{definition}

\newcommand{\Input}{\item[{\bf Input:}]}
\newcommand{\Output}{\item[{\bf Output:}]}
\renewcommand{\Return}{\item[{\bf return}]}

\newenvironment{proofof}[1]{{\em Proof of #1.}}{\hfill
\qed}
\usepackage{tikz}

\numberwithin{equation}{section}

\title{A Polynomial Time MCMC Method for Sampling from Continuous DPPs}

\author{
\begin{tabular}{ccccc}
      {\begin{tabular}{c}Shayan Oveis Gharan \\University of Washington \\\small{shayan@cs.washington.edu } \end{tabular}} & & & &
      {\begin{tabular}{c}Alireza Rezaei \\University of Washington \\\small{arezaei@cs.washington.edu} \end{tabular}}
\end{tabular} 
  }

\def\M{\mathcal{M}}
\def\ovn{\overline{n}}

\def\mP{\mathbb{P}}

\newcommand{\frsc}[3]{CD$_#3(#1,#2)$}
\newcommand{\srsc}[2]{CD$(#1,#2)$}

\newcommand{\norm}[1]{\left\|#1\right\|}

\newcommand{\R}{{\mathbb R}}
\newcommand{\poly}{\text{poly}}

\newcommand{\EE}[1]{{{\mathbb{E}}\left[{#1}\right]}}
\newcommand{\EEE}[2]{\mathbb{E}_{#1} #2}
\newcommand{\PP}[2]{\mP_{#1}\left[#2\right]}
\newcommand\restr[2]{{
  \left.\kern-\nulldelimiterspace 
  #1 
  \vphantom{\big|} 
  \right|_{#2} 
  }}
\DeclareMathOperator{\Var}{var}
\DeclareMathOperator{\TV}{TV}

\DeclareMathOperator{\Vol}{vol}
\DeclareMathOperator{\Tr}{tr}

\begin{document}
\maketitle
\begin{abstract}
	We study the Gibbs sampling algorithm for continuous determinantal point processes. We show that, given a warm start, the Gibbs sampler  generates a random sample from a continuous $k$-DPP defined on a $d$-dimensional domain  by only taking $\poly(k)$ number of steps. As an application, we design an algorithm to generate random samples from $k$-DPPs defined by a spherical Gaussian kernel on a unit sphere in $d$-dimensions, $\mathbb{S}^{d-1}$ in time polynomial in $k,d$. 
\end{abstract}

\section{Introduction}
Let $L\in \R^{n\times n}$ be a positive semi-definite (PSD) matrix. A discrete determinantal point process  with \textit{kernel} $L$ is a probability distribution $\mu:2^{[n]}\to \R_+$ defined by
$$\mu(S) \propto \det(L_S), \, \forall S \subseteq [n].$$
The notion of DPP was first introduced by \cite{macchi1975coincidence} to model fermions.
Since then, They have been extensively studied, and efficient algorithms have been discovered for tasks like sampling from DPPs \cite{li2015efficient,deshpande2010efficient,anari2016monte}, marginalization \cite{borodin2005eynard}, and learning \cite{gillenwater2014expectation,urschel2017learning} them (in the discrete domain). In machine learning they are mainly used to solve problems where selecting a \textit{diverse} set of objects is preferred since they offer negative correlation. To get intuition about why they are good models to capture diversity, suppose each row of the gram matrix  associated with the kernel is a feature vector representing an item. It means the probability of a set of items is proportional to the square of the volume of the space spanned the vectors representing items. Therefore, larger volume shows those items are more spread which resembles diversity. Text summarization, pose estimation, and diverse image selection are examples of applications of DPP in this area  \cite{gillenwater2012discovering,kulesza2012determinantal,kulesza2010structured,kulesza2011k}. These distributions also naturally appear in many contexts including  non-intersecting random walks \cite{johansson2002non},
 random spanning trees \cite{burton1993local}.
 
  Here, we focus on the sampling problem of DPPs. In the discrete setting, 
 the first sampling algorithms was proposed by \cite{hough2006determinantal}. They propose a two-step spectral algorithm
 which generates a random sample by running an eigen-space decomposition of the kernel and running conditional sampling on the space spanned by a randomly chosen set of the eigenvectors.
 Several algorithms for sampling different variation of DPPs including $k$-DPPs have been built on this idea \cite{kulesza2012determinantal,deshpande2010efficient}. The disadvantage of spectral techniques for this problem is that they typically need the eigen-decompostion or Cholesky decomposition of the kernel which makes them inefficient for large instances. Recently, several group of researchers studied employing the Monte Carlo Markov Chain (MCMC) technique for this task \cite{anari2016monte,li2016fast,rebeschini2015fast}. It is shown in \cite{anari2016monte} that the  natural Metropolis-Hastings algorithm for $k$-DPPs gives an efficient sampling method running in time $O(n)\poly(k)$, where $n$ is the number points in the underlying kernel.

In this paper, our main goal is to study the sampling problem of continuous DPPs. On a continuous domain, a DPP 
is defined similarly by a continuous PSD operator. 
For $C \subseteq{R}^d$, and a continuous PSD kernel $L:C\times C \to \R$, the DPP 
is a distribution over finite subsets of $C$ where for every subset $S$, the probability density function at $S$, $p(S)$ satisfies
 $$p(S) \propto \det(L_S).$$
 For an integer $k>0$, a (continuous) $k$-DPP is the restriction of a (continuous) DPP to subsets of size exactly $k$.
 Continuous DPPs naturally arise in several areas of Phyiscs, Math and Computer Science;
 To name a few examples,  eigenvalues of random matrices \cite{mehta1960density,Gin65}, zero-set of Gaussian analytic functions \cite{peres2005zeros} are families of DPPs; also,  see \cite{LMR15} for applications in statistics and \cite{BNL16} for connections to repulsive systems. Recently, sampling from continuous $k$-DPP has also been used for tuning the hyper-parameters of a network \cite{dodge2017open}.
 Unlike the discrete setting, despite several attempts \cite{hafiz2013approximate,scardicchio2009statistical,lavancier2012statistical,bardenet2016monte,hennig2016exact}, to this date, we are not aware of any efficient sampling  algorithms with provable guarantees. It remains an open problem to design an efficient sampling algorithm for continuous $k$-DPPs.

Our main contribution is to develop an algorithm to draw approximate samples from a  $k$-DPP defined by a continuous kernel, and having access to a ``conditional-sampling'' oracle.

Unlike \cite{anari2016monte}, here, we analyze a different Markov chain, called the {\em Gibbs sampler chain}: Let $\pi$ be a $k$-DPP defined by a kernel $L:C \times C \to \R$ for $C\in \R^d$.
Given a state $\{x_1,\dots,x_k\}$, the  Gibbs sampler $\M$  moves as follows: Remove a point $x_i \in \{x_1,\dots,x_k\}$ is chosen  uniformly at 
random, and move to the state $\{x_1,\dots,x_{i-1},y,x_{i+1},\dots,x_k\}$ with probability proportional to
$\det_L (x_1,\dots,x_{i-1},y,x_{i+1},\dots,x_k)$.

 \subsection{Results}
 \label{sec:results}
We study the problem of sampling from a continuous $k$-DPP, and present the first MCMC based algorithm with provable guarantees for this problem.
Our main contribution is to show that the Gibbs 
sampler for $k$-DPPs (see \autoref{def:Gibbs})   
\textit{mixes} rapidly, and can be simulated efficiently   
under some extra assumptions on the kernel.
More precisely, we analyze the 
\emph{conductance} of the Gibbs samplers $k$-DPPs (see theorems 
\ref{thm:mainthmdiscretechain} and 
\ref{thm:ergodicflow}), and using the well-known connection between the conductance and mixing 
time obtain the following.
\begin{theorem}
\label{thm:mainresult1}
Let $\M$ be the Gibbs sampler for a $k$-DPP $\pi$. If we run the chain starting from an arbitrary distribution $\mu_0$, for any $\epsilon >0$ we have
\begin{equation*}
    \tau_{\mu_0}(\epsilon) \leq O(k^4)\cdot\log\left(\frac{\Var_\pi(\frac{f_{\mu_0}}{f_\pi})}{\epsilon }\right).
\end{equation*}
\end{theorem}
In the above theorem, $f_\pi$ and $f_{\mu_0}$ refer to the 
probability density  functions for $\pi$ and $\mu_0$, 
respectively. Moreover, $\tau_{\mu_0}(\epsilon)$ denotes the mixing time for the chain started from $\mu_0$, and is defined by 
$$\tau_{\mu_0}(\epsilon) = \min \{t \, | \, d_{\TV}(\mu_t,\pi)  \leq \epsilon \},$$
where $\mu_t$ is the distribution after $t$ steps, and $d_{\TV}$ denotes the total variation distance.
Moreover, an $\epsilon$-approximate sample for distribution $\pi$ refers to a random sample from a distribution $\mu$ where $d_{\TV}(\mu,\pi) \leq \epsilon$.

To find a ``good'' starting 
distribution $\mu$ for which the  
bound in \autoref{thm:mainresult1} is polynomial, we require some 
additional constraints on the kernel. Namely, we need to have access to   \emph{conditional sampling oracles}, formally defined as follows.

\begin{definition} 
\label{def:CDdist}
For a kernel $L:C \times C\to \R$, a subset $S \subset C$, and an integer $j$,  we define $(S,j)$-conditional distribution of $L$ to be a simple point process defined on $\binom{C}{j}$ by a  pdf function $f$ satisfying
$$\forall \{x_1,\dots,x_j\}\subset C: \hspace{1mm} f(x_1,\dots,x_j) \propto \det_L(S\cup \{x_1,\dots,x_j\}),$$
and zero if $S \cap \{x_1,\dots,x_j\}\neq \emptyset$. We denote this distribution by \frsc{S}{j}{L}.
We say an algorithm is an \frsc{i}{j}{L} oracle for integers $i$ and $j$, if given any $S \subset C \,\,(|S|=i)$, it returns a sample from the \frsc{S}{j}{L}.
\end{definition}
It is straight-forward to see that taking a step of the Gibbs sampler of 
the $k$-DPP from the state $x_1,\dots,x_k$ defined by $L$ is equivalent to removing a point $x_i$, for some $1\leq i\leq k$, and  generating a 
sample from \frsc{\{x_1,\dots,x_{i-1},x_{i+1},\dots,x_{k}\}}{1}{L}.
\noindent We prove the following.
\begin{theorem} [informal]
\label{thm:mainresult2}
Let  $\M$ be the Gibbs sampler for the $k$-DPP defined by a kernel $L$. Given \frsc{i}{1}{L} oracles for all $0\leq i \leq k-1$, we sample a starting state for $\M$ from
a probability distribution $\mu$ where    $$\tau_\mu(\epsilon) \leq O(k^5\log \frac{k}{\epsilon}).$$
\end{theorem}

Therefore, to get a polynomial time algorithm for sampling from a  $k$-DPPs (a CD$(0,k)$ distribution), it is essentially enough 
to have efficient algorithms to sample from conditional $1$-DPPs  (CD$(.,1)$
distributions), which is a much simpler problem.
As an application, we consider 
Gaussian kernels. For a covariance 
matrix $\Sigma \in \R^{d\times d}$, a Gaussian kernel $\mathcal{G}_\Sigma$  is defined by $\mathcal{G}_\Sigma(x,y)=
\exp(-(x-y)^\intercal \Sigma^{-1}(x-y))$. We show a simple rejection sampling can be used as conditional sampling oracles for $\mathcal{G}_{\sigma I}$, and obtain the following.
\begin{theorem}[See \autoref{thm:Gaussian} for details.]
\label{thm:mainresult3}
Let $d$, and $k\leq e^{d^{1-\delta}}$   for some 
$0<\delta<1$ be two integers.  There is a randomized 
algorithm that for any $\epsilon >0$ and $\sigma\leq 1$, 
generates an $\epsilon$-approximate sample from the 
$k$-DPP defined by $\mathcal{G}_{\sigma I }$ restricted to 
$\mathbb{S}^{d-1}$ which runs in time $O(d\log  \frac{1}{\epsilon})\cdot k^{O(\frac{1}{\delta})}.$ 
\end{theorem}
In the above, we are assuming a sample from the normal distribution can be generated in constant time.
\subsection{Previous Work}
\label{relatedwork}
 In the continuous regime, the efforts have been mostly concentrated on finding the eigen-decomposition of the kernel or a low-rank approximation of it, and extending the aforementioned spectral techniques to the continuous space \cite{hafiz2013approximate,scardicchio2009statistical,lavancier2012statistical,bardenet2016monte}. 
 However, in theory, these methods does not yield provable guarantees for sampling because generally speaking, to project the DPP
kernel onto a lower dimensional space, they   minimize the error with respect to a matrix norm, rather than the DPP distribution.
 Moreover, there are two main obstacles to implement this approach: First, there is no efficient algorithm for obtaining  an eigen-decomposition of a kernel defined on an infinite space because it may have infinitely many eigenvalues. Secondly, for a general continuous eigen-decomposition, it is impossible to run a conditional sampling algorithm. For a concrete example,  \cite{lavancier2012statistical} uses an orthonormal Fourier transform to find a low-rank approximation of the DPP kernel, and then proceeds with conditional sampling via \textit{rejection} sampling. 
 But, they do not provide any rigorous guarantee on the distance between the resulting distribution, and the underlying $k$-DPP distribution. A similar spectral approach is proposed in \cite{hafiz2013approximate}. They consider the Nyst\"orm method, and random Fourier feature as two techniques to find low-rank approximation of the kernel. The approximation 
 scheme that they use enables them to handle a wider range of 
 kernels. However, the first 
 issue still remains unresolved: 
 as the rank of the approximated 
 kernel increases, the resulting 
 distribution becomes closer to 
 the initial $k$-DPP, but to the 
 best of our knowledge there is 
 no provable guarantee. 
 \cite{hafiz2013approximate} also
 provides empirical evidence that Gibbs sampling is efficient to
 generate sample from continuous 
 $k$-DPP in many cases. However, 
 they do not provide any rigorous 
 justification.
It is also worth mentioning that 
\cite{hennig2016exact} claims to 
devise an algorithm to generate 
exact samples for specific 
kernels (including Gaussian), 
yet a careful look at their 
method would reveal a major flaw 
in their argument \footnote{The 
distribution that they consider 
as the conditional distribution 
of the $k$-DPP is in fact 
equivalent to our notion of 
conditional distribution of the 
kernel (see \autoref{def:CDdist} )}. 

\subsection{Techniques}
Our first contribution is to analyze the Gibbs sampler chain in the discrete setting. We prove for a $k$-DPP defined on $n$ points, the spectral gap of the Gibbs sampler chain is a polynomial in $1/k$ and {\em independent} of $n$. So, up to logarithmic factors in $n$, the chain mixes in time polynomial in $k$. This result on its own could be of interest in designing distributed algorithms for sampling from discrete $k$-DPPs. This is because given access to $m$ processors, one can generate the next step of the Gibbs sampler in time $O(n/m)$. 

Secondly, we lift the above proof to the continuous setting using a natural discretization of the underlying space. To prove the mixing time, we need to make sure that the logarithm of the variance of the starting distribution with respect to the stationary distribution of the chain, i.e., the $k$-DPP, is polynomially small in $k,d$. We use a simple randomized greedy algorithm for this task: We start from the empty set; assuming we have chosen $x_1,\dots,x_i$ we sample $x_{i+1}$ from $CD_L(\{x_1,\dots,x_i\},1)$, where as usual $L$ is the underlying kernel. 
We show that the distribution governing the state output by this algorithm is our desired starting distribution.

Lastly, we use our main theorem to generate samples from a $k$-DPP defined on a spherical Gaussian kernel on $\mathbb{S}^{d-1}$. To run the above algorithm we need to construct the $CD_L(i,1)$ for all $0\leq i\leq k-1$ where $L$ is the corresponding kernel. Given the point $\{x_1,\dots,x_i\}$, we use the classical rejection sampling algorithm to choose $x_{i+1}$; namely, we generate a uniformly random point on the unit sphere and we accept it with probability $\frac{\det_L(x_1,\dots,x_{i+1})}{\det_L(x_1,\dots,x_i)}$. We use the distribution of the eigenvalues of the spherical Gaussian kernel \cite{minh2006mercer} to bound the expected number of proposals in the rejection sampler.

\section{Preliminaries}
Let $\R^d$ denote the $d$-dimensional 
euclidean space. Whenever, we consider
$C \subset \R^d$ as measurable space,
our measure  is the standard Lebesgue 
measure. 
The $\Vol(C)$ denotes the $d$-dimensional volume of 
$C$ with respect to the standard 
measure. 
A function $f:C\to R$ belongs to $\ell^2(C)$, if $\int_C |f(x)|^2 dx < \infty$. For two such functions $f,g$ the standard inner product is defined by $\langle f,g \rangle \int_C f(x)g(x) dx$. A function $L: C \times 
C \to \R$, is a Hilbert-Schmidt 
kernel $\int_C \int_C |L(x,y)| dx dy < \infty$. The associated 
Hilbert-Schmidt integral operator is 
then a  linear operator which for 
any $f \in \ell^2(C)$  is defined by 
$ \forall x \in C: \hspace{1mm} Lf(x) = \int_C L(x,y)f(y)dy.$
The operator is self-adjoint if for 
any $f,g \in \ell^2(C)$, $\langle f,Lg \rangle =\langle Lf,g \rangle$. It is
Positive Semi-Definite (PSD), if for
any $f$, $\langle Lf,f \rangle \geq
0$. 

\paragraph{Mercer's Conditions.} A kernel $L$ satisfies the Mercer conditions if: $L$ is symmetric, which means for any $x,y \in C$, $L(x,y)=L(y,x)$. Moreover, for any finite 
sequence $x_1,\dots,x_i \in C$, the submatrix 
$L_{\{x_1,\dots,x_i\}}$ is a PSD matrix. 
It is known that the operators 
satisfying Mercer conditions are PSD. 
Moreover, if $L$ satisfies Mercer's 
condition, for any $x \in C$, there 
exists a Hilbert space $H$, and a 
function $f_x: H\to \R$, where for any $y \in C$, $L(x,y)=\langle f_x,f_y
\rangle_H$. These functions are also known as feature maps. 
We also use the classical 
Mercer's theorem  which states that operators satisfying the Mercer's condition are compact, and so have a 
countable system of eigen-spaces and eigenvalues. i.e. there 
are non-negative eigenvalues $\lambda_1,\lambda_2\dots$, and 
$\{\phi_i\}_{i=1}^\infty \subset \ell^2(C)$ where 
$$L = \sum_{i=1}^\infty \lambda_i \phi_i(x) \phi_i(y).$$ In 
section \ref{sec:app}, we use this result for Gaussian 
kernels.
Throughout, the rest of the paper, whenever we say a continuous kernel, we refer to continuous Hilbert-Schmidt kernel which  satisfies the Mercer's conditions.

If $\pi$ is a probability distribution, we use $f_\pi$ to 
refer to the corresponding probability 
density function (pdf).  We use bold 
small letters to refer to a finite set 
of points in $\R^d$, and in particular a state of the Gibbs sampler for a
$k$-DPP, e.g. $\bm{x}=\{x_1,\dots,x_k\}
\subset \R^d$. For $y \in \R^d$, we may use $\bm{x}+y$ to indicate $\bm{x}\cup \{y\}$.  
For any   
$\bm{x}=\{x_1,\dots,x_k\}$, we use 
$\det_L(x_1,\dots,x_k)$ and 
$\det_L(\bm{x})$ interchangeably to 
refer to determinant of the $k\times k$ 
submatrix where the $ij_{\text{th}}$ 
entry is $L(x_i,x_j)$. Whenever, the 
kernel is clear from the context, we may drop the subscript. 
For two expression $A$ and $B$, we write $A \lesssim B$ to denote $A \leq O(B)$. 
\subsection{Continuous Determinantal Point Process} 
A Determinantal Point Process (DPP) on a finite set, namely
$[n]$ is a probability distribution $\pi$ on the subsets of
$[n](2^{[n]})$ which is defined by a PSD matrix (a.k.a
\emph\emph{kernel}) $L\in \R^{n \times n}$ where for every
subset $S \subset [n]$, 
$$\mathbb{P}(S) \propto \det(L_S)$$
where $L(S)$ is the principal submatrix of $L$ indexed by elements of $S$.
For an integer $0 \leq k \leq n$, the restriction of $\pi$ to subsets of size $k$ is called a $k$-DPP defined by the 
kernel $L$. So support of a $k$-DPP defined on $[n]$ is 
$\binom{[n]}{k}$. 

Similarly a continuous $k$-DPP can be 
defined on a continuous domain with a  
continuous PSD kernel $L:\R^d\times 
\R^d \to R$. The above succinct 
definition
suffices to understand the rest of the paper. However, for
completeness, we formally define them in the following. For
more details about DPPs and generally  point processes on continuous domains, we refer
interested readers to \cite{hough2006determinantal}.
\paragraph{Continuous $k$-DPP.}
For a kernel $L: \R^d \times
\R^d \to \R$, and for an integer $k$, 
the $k$-DPP defined by $L$ on domain 
$C$ is a point process with the support of subsets of $C$ of size $k$,
$\binom{C}{k}$, defined as
follows: 
For any $\{x_1,\dots,x_k\} \subset C$ the probability density
function is proportional to $\det(x_1,\dots,x_k)$. i.e. for
any mutually disjoint family of subsets $D_1,\dots,D_k
\subset C$,
$$\pi \left( \left \{\{x_1,\dots,x_k\} \middle|  
\forall i, \, \in D_i  \right\}\right) = 
\frac{1}{Z}\int_{D_1}\dots\int_{D_k} 
\det_L(x_1,\dots,x_k) dx_k\dots dx_1,$$
where $Z$ is the partition function $Z=\frac{1}{k!} \int_C\dots
\int_C \det_L(x_1,x_2,\dots,x_k)dx_k\dots 
dx_1$.
Now 
As alluded to before, we study the Gibbs sampling scheme for a $k$-DPP which is formally  defined as follows:
 \begin{definition}[Gibbs samplers for $k$-DPPs]
 \label{def:Gibbs}
Let $\pi$ be a $k$-DPP defined by a kernel $L:C \times C \to \R$ for $C\in \R^d$.
The  Gibbs sampler $\M$ for $\pi$
is a Markov chain with state
space $\binom{C}{k}$ and 
stationary measure $\pi$ 
which moves as follows: Let $\{x_1,\dots,x_k\} \subset C$ be the current state. A point $x_i \in \{x_1,\dots,x_k\}$ is chosen  uniformly at 
random, and the chain moves to the state $\{x_1,\dots,x_k\}-x_i+y$ for $y \in C$ chosen by the distribution defined by the pdf function
\begin{equation}\label{eq:gibbssample}
f(y):=  \propto \det_L (x_1,\dots,x_{i-1},y,x_{i+1},\dots,x_k)
\end{equation}
\end{definition}

\subsection{Markov Chains with Measurable State Space}
In this section we give a high level overview on the 
theory of Markov chains defined over measurable sets. We refer interested readers to \cite{lovasz1993random} for more details.  Let $(\Omega,\mathcal{B})$ be a
measurable space. In the most general setting, a
Markov chain is defined by the triple $(\Omega,\mathcal{B},\{P_x\}_{x \in \Omega})$, where
for every $x \in \Omega$, $P_x:\mathcal{B}\to \R_{+}$ is a probability measure on $(\Omega,\mathcal{B})$.
Also,
 for every fixed $B \in \mathcal{B}$, 
 $P_{x}(B)$ is a measurable function in terms 
 of $x$. In this setting starting from a 
 distribution $\mu_0$, after one step the 
 distribution $\mu_1$ would be given by
 $$ \mu_1(B) = \int_{\Omega} P_x(B)d\mu_0(x), 
 \, \forall B \in \mathcal{B}.$$ 
 From now on, assume $\Omega \subset \R^k$ and $\mathcal{B}$ is the standard Borel $\sigma$-algebra.
 In our setting, we can assume the transition probabilities are given by a 
kernel \textit{transition kernel}
$P:\Omega\times\Omega \to \R_+$ 
where for any measurable $A \subset \Omega$,  we can write
 $$P_x(A)  = \int_{A} P(x,y)dy.$$
 In this notation, we use 
 $P(x,B)$ and $P_x(B)$ 
 interchangeably. $P^n(x,.)$ 
 would also denote the 
 probability distribution of the 
 states after $n$ steps of the chain started at $x$.
 Similar to the discrete setting, we can define the \textit{stationary measure} for the chain. A probability distribution $\pi$ on $\Omega$ is stationary if and only if for every measurable set $B$, we have
 $$ \pi(B) = \int_{\Omega} \int_{B} P(x,y) dyd\pi(x).$$
 We call $\M$ 
 $\phi$-\textit{irreducible} for a 
 probability measure $\phi$ if for any  set $B\in 
 \mathcal{B} $ with $\phi(B)>0$, and any state $x$,  there is $t \in \mathbb{N}$
 such that $P^t(x,B)>0$. It is called
 \textit{strongly $\phi$-irreducible} 
 if for any $B\subseteq \Omega$ with
 non-zero measure and $x\in \Omega$,
 there exists $t\in \mathbb{N}$ such
 that for any $m\geq t$, $P^m(x,B)>0$. 
 We say $\M$ is {\em
 reversible} with respect to a measure
 $\pi$ if for any two sets $A$ and $B$ 
 we have
 $$ \int_{B} \int_A P(y,x)dx d\pi(y) = \int_{A} \int_B P(x,y) dy d\pi(x).$$
 In particular, reversibility 
 with respect to a measure, 
 implies it is a stationary 
 measure. Is is immediate from this to verify that 
 for a Gibbs sampler of a 
 $k$-DPPs $\pi$, the $\pi$ itself is 
 the stationary measure. 
 Moreover, if the kernel of the 
 $k$-DPP is continuous, it is 
 straight-forward to see that it 
 is $\pi$-strongly irreducible.
 The following lemma also shows 
 $\pi$ is the unique stationary 
 measure, and as the number of 
 steps increases, the chain 
 approaches to the unique 
 stationary measure.
  \begin{lemma}[\cite{diaconis1997markov}]
 \label{lem:uniquestationary}
 If $\pi$ is  a stationary measure of $\M$, and $\M$ is strongly $\pi$-irreducible. Then for any other distribution $\mu$ which is absolutely continuous with respect to $\pi$, $\lim_{n \to \infty} |P^n(\mu,.)-\pi|_{\text{TV}}=0$.
 \end{lemma}
 From now on, consider $\M=(\Omega,P,\pi)$ is  chain with state space $\Omega$, probability transition function $P$, and  a unique stationary measure $\pi$. 
Let us describe some results about mixing time in the Markov chains defined on continuous spaces. But before that we need to setup some notation.
Consider a Hilbert space $\ell^2(\Omega,\pi)$ equipped with the following inner product.
$$ \langle  f,g \rangle_\pi = \int_\Omega f(x)g(x)d\pi(x).$$
$P$ defines an operator in this space where for any function $f \in \ell^2(\Omega,\pi)$ and $x\in\Omega$,
$$(Pf)(x) = \int_\Omega P(x,y)f(y)dy.$$ 
In particular $\M$ being reversible is equivalent to $P$ being self-adjoint. For a reversible chain $\M$ and a function $f \in \ell^2(\Omega,\pi)$, the Dirichlet  form $\mathcal{E}_P(f,f)$ is defined as
$$\mathcal{E}_P(f,f) = \frac{1}{2}\int_\Omega \int_\Omega  (f(x)-f(y))^2 P(x,y)d\pi(x)dy.$$
We also define the \textit{Variance} of $f$ with respect to $\pi$ as
$$\Var_\pi(f):= \int_\Omega (f(x)-\mathbb{E}_\pi(f))^2d\pi(x).$$ 
We may drop the subscript if the underlying stationary distribution is clear in the context.
One way for upperbounding the mixing time of a chain is to use is to its spectral gap which is also known as \textit{Poincar\'e Constant}.
\begin{definition} [Poincar\'e Constant]. The Poincar\'e constant of the chain is defined as follows,
\begin{equation*}
 \lambda := \inf_{f:\pi\to \mathbb{R}}\frac{\mathcal{E}_P(f,f)}{\Var(f)},
\end{equation*}
where the infimum is only taken over all functions in $\ell^2(\Omega,\pi)$ with non-zero variance.
\end{definition}
In this paper, we use the following theorem to upperbound the mixing time of the chain relevant to us.
\begin{theorem}[\cite{Meyn2012geometric}]\label{thm:markovchainmixing}
For any reversible, strongly 
$\pi$-irreducible Markov chain 
$M=(\Omega,P,\pi)$, if $\lambda>0$, 
then the distribution of the chain started from $\mu$ $($which is absolute continuous with respect to $\pi)$  is
$$\|P^t(\mu,.) - \pi \|_{TV} \leq \frac{1}{2} (1-\lambda)^t \sqrt{\Var\left(\frac{f_\mu}{f_\pi}\right)}.$$
\end{theorem}
For the sake of completeness, we include a proof of the above theorem which is an extension of the proof of the analogous discrete result in \cite{fill1991eigenvalue}. We need the following simple lemma known as Mihail's identity. 
\begin{lemma}[Mihail's identity, \cite{fill1991eigenvalue}]
For any reversible irreducible Markov chain $\M=(\Omega,P,\pi)$, and any function $f$ in $L^2(\pi)$,
$$ \Var(f) = \Var(Pf) + \mathcal{E}_{P^2}(f,f).$$
\end{lemma}
\begin{proofof}{\autoref{thm:markovchainmixing}}
First of all, one can easily verify that if a chain is lazy and irreducible, then it is strongly-irreducible. Combining it with  \autoref{lem:uniquestationary} would guarantee the uniqueness of the stationary measure. Let $\mu_0=\mu$ be the starting distribution and define $\mu_t=P^t(\mu,.)$ be the distribution at time $t$. Set $f_t:=\frac{f_{\mu_t}}{f_\pi}$, we have
\begin{equation*} (Pf_t)(x) = \int_\Omega P(x,y)\frac{f_{\mu_t}(y)}{f_\pi(y)} dy = \int_\Omega \frac{P(y,x)f_{\mu_t}(y)}{f_\pi(x)} dy = \frac{f_{\mu_{t+1}}}{f_\pi}(x)=f_{t+1}(x)
\end{equation*}
which implies 
\begin{equation}\label{eq:varchangeinonestep}
\Var(Pf_t)= \Var(f_{t+1}) 
\end{equation}

So applying Mihail's identity on $\frac{f_{\mu_n}}{f_\pi}$ and using \eqref{eq:varchangeinonestep} , we conclude 
\begin{equation}\label{eq:mihail} 
\Var(f_t) = \Var(f_{t+1}) + \mathcal{E}_{P^2} (f_t,f_t).
\end{equation}
Now, note that $P^2$ has the same stationary distribution $\pi$, so its Poincar\'e constant is at most 
$$ \lambda(P^2) \leq \frac{\mathcal{E}_{P^2} (f_t,f_t)}{\Var(f_t)}.$$
Combining this with \eqref{eq:mihail}, and using induction we can deduce
$$\Var(f_t) \leq (1-\lambda(P^2))^{t}\Var(f_0).$$
Note that, since $P$ is the kernel for a lazy chain, it has no negative values in its spectrum, implying $1-\lambda(P^2)= (1-\lambda(P))^2$.  So in order to complete the proof it is enough show
$$ 4\|\mu_t - \pi \|_{TV}^2 \leq \Var(f_t).$$
This can be seen using an application of Cauchy-Schwarz's inequality. We have
\begin{align*}
    4\|\mu_t - \pi \|_{TV}^2  &= \left( \int_\Omega |f_{\mu_t}(x) - f_\pi(x)| dx\right)^2 \\
    &= \left(\int_\Omega f_\pi(x)\left|\frac{f_{\mu_t}(x)}{f_\pi(x)} - 1\right|dx\right)^2 \\
    &\leq \int_{\Omega} f_\pi(x) \left|\frac{f_{\mu_t}(x)}{f_\pi(x)}-1\right|^2dx  =  \Var(\frac{f_{\mu_{t}}}{f_\pi})
\end{align*}
The last identity uses that $\EEE{\pi}{\frac{f_{\mu_t}}}{{f_\pi}}=1$.
This completes the proof.
\end{proofof}

In order to take advantage of \autoref{thm:markovchainmixing}, we need to lowerbound the Poicar\'e constant of our chain. This can be done by lowerbounding the \textit{Ergodic Flow} of the chain.
\begin{definition}[Ergodic Flow] For a chain $\M=(\Omega,P,\pi)$, the ergodic flow $Q:\mathcal{B}\to [0,1]$ is defined by
$$ Q(B) = \int_B \int_{\Omega\setminus{B}}P(u,v)dv f_\pi(u)du.$$
\end{definition}
\noindent The \textit{conductance} of a set $B$ is defined by, $\phi(B):= \frac{Q(B)}{\pi(B)}$, and the conductance of the chain is
$$\phi(\M) = \min_{0 <\pi(B)\leq \frac{1}{2}} \phi(B).$$
The following theorem which is an extension of the Cheeger's inequality for the Markov chains on a continuous space, relates the spectral gap to conductance.
\begin{theorem}[\cite{lawler1988bounds}]\label{thm:cheegerinequality}
For a chain $\M$ defined on a general state space with spectral gap $\lambda$ we have
$$\frac{\phi(\M)^2}{8} \leq \lambda \leq 2\phi(\M).$$
\end{theorem}
\section{Gibbs Sampling for Discrete $k$-DPP}
In this section we prove the Gibbs sampler for a discrete $k$-DPP is an $\Omega\left(\frac{1}{k^2}\right)$-expander. Recall that the conductance of  a time reversible   chain $\mathcal{M}=(\Omega,P,\pi)$ is defined by $$\Phi(\mathcal{M}) = \min_{S \subset \Omega:\pi(S) \leq \frac{1}{2}} \frac{Q(S,\overline{S})}{\pi(S)},$$
where for $x,y \in \Omega$, $Q(y,x)=Q(x,y)= \pi(x)P(x,y)$.
We prove the following.
\begin{theorem}
\label{thm:mainthmdiscretechain}
Let $\M$ be the Gibbs sampler chain  for an arbitrary discrete $k$-DPP, then for a constant $C$ we have
\begin{equation*}
    \phi(\M) \geq \frac{1}{Ck^2}
\end{equation*}
\end{theorem}
In the rest of this section, we fix $\M=(\Omega,P,\pi)$ to be the Gibbs-sampler chain on a $k$-DPP defined on a set of $n$ elements. 

Before discussing the details of the proof let us first fix a notation and recall fundamental properties of $k$-DPPs.
For any element $1\leq i\leq n$, define $\Omega_i,\Omega_{\overline{i}}$ be the set of all states in 
$\Omega$ that contain, do not contain $i$, respectively.
Also define
\begin{equation*}
    \begin{aligned} 
    \pi_i &:= \{ \pi| \text{ $i$ is chosen }\},\text{ i.e. } \pi_i(\bm{x}) = \frac{\pi(\bm{x})}{\pi(\Omega_i)}, \forall x \in \Omega_i \\
    \pi_{\overline{i}} &:= \{ \pi| \text{ $i$ is not chosen }\},\text{ i.e. }
    \pi_{\overline{i}}(\bm{x}) = \frac{\pi(\bm{x})}{\pi(\Omega_{\overline{i}})}, \forall x \in \Omega_{\overline{i}}
    \end{aligned}
.\end{equation*}
It follows from \cite{anari2016monte} that $\pi_i,\pi_{\overline{i}}$ can be identified with a $(k-1)$-DPP, $k$-DPP   supported on $\Omega_i, \Omega_{\overline{i}}$, respectively. We define
 $\M_i=(\Omega_i,P_i,\pi_i),\M_{\overline{i}}=(\Omega_{\overline
{i}},P_{\overline{i}},\pi_{\overline{i}})$ to be the \emph{restricted} Gibbs samplers.
So, it is straightforward to see that for any $x,y \in \Omega_i$ we get $P_i(\bm{x},\bm{y})= \frac{k}{k-1}\cdot P(\bm{x},\bm{y})$.
and consequently for $Q_i$ defined  as $Q$ for $\M_{i}$, we get 
\begin{equation}
\label{eq:edgeinchainn}
Q_i(\bm{x},\bm{y})=\frac{Q(\bm{x},\bm{y})}{\pi(\Omega_i) }
.\end{equation} 
Unlike $P_i$,  $P_{\overline{i}}$  is not obtained from 
scaling a restriction of $P$. In particular, Let 
$\bm{x},\bm{y} \in \Omega_{\overline{i}}$ so that 
$P_{\overline{i}}(\bm{x},\bm{y})>0$ (which implies 
$|\bm{x}\cap \bm{y}|=k-1$). Then, setting $I=\bm{x}\cap 
\bm{y}$ and  with a bit abuse of notation $\pi(I)= \sum_{j 
\in [n]\setminus I} \pi(I+j)$, i.e. $\pi(I)= \mP_{\bm{z} 
\sim \pi}[I \subset \bm{z}]$, we have   
\begin{align}
%
P_{\overline{i}}(\bm{x},\bm{y})&= \frac{1}{k} \cdot \frac{\pi(\bm{y})}{\pi(I) - \pi(i+I)}
\end{align}
whereas $P(\bm{x},\bm{y})= \frac{\pi(\bm{y})}{k\cdot \pi(I)}$. 
For any $\bm{x} \in \Omega_i$, define $N_{\overline{i}}(\bm{x})$ be the set of its neighbours in $\Omega_{\overline{i}}$, i.e. $$N_{\overline{i}}(\bm{x}) = \{\bm{y} \in \Omega_{\overline{i}} | \vspace{1mm} P(\bm{x},\bm{y})>0\}.$$ We use the following lemma to relate $Q_{\overline{i}}$ to $Q$.
\begin{lemma}
\label{lem:induconbarn}
Let $A \subset \Omega_{\overline{i}}$ be an arbitrary subset. For
a  state $\bm{x}\in \Omega_i$, consider the following 
partitioning of $N_{\overline{i}}(\bm{x})$: $N_A= 
N_{\overline{i}}(\bm{x})\cap A$ and
$N_{\overline{A}}=N_{\overline{i}}(\bm{x})\cap
(\Omega_{\overline{i}}\setminus A)$. 
Then we have 
\begin{equation}
    \label{eq:auxinduct}
    Q(\bm{x},N_A)+Q(N_A,N_{\overline{A}})\geq  \pi(\Omega_{\overline{i}})\cdot Q_{\overline{i}}(N_A,N_{\overline{A}}).
\end{equation}
\end{lemma} 
\begin{proof}
Note that $\bm{x} \cup N_A \cup N_{\overline{A}}$ is the set of all states containing elements in $\bm{x}-i$. So by definition of $Q$ and $Q_{\overline{i}}$, we have 
\begin{align}
    Q(\bm{x},N_A)+Q(N_A,N_{\overline{A}})&=\frac1{k}\cdot \frac{\pi(\bm{x})\pi(N_A)}{\pi(\bm{x})+\pi(N_A)+\pi(N_{\overline{A}})}+\frac1{k}\cdot \frac{\pi(N_{\overline{A}})\pi(N_A)}{\pi(\bm{x})+\pi(N_A)+\pi(N_{\overline{A}})}\\ &=\frac{\pi(N_A)}{k}\cdot \frac{\pi(\bm{x})+\pi(N_{\overline{A}})}{\pi(\bm{x})+\pi(N_A)+\pi(N_{\overline{A}})} \geq \frac{\pi(N_A)}{k}\cdot \frac{\pi(N_{\overline{A}})}{\pi(N_A)+\pi(N_{\overline{A}})} = \pi(\Omega_{\overline{i}})\cdot Q_{\overline{i}}(N_A,N_{\overline{A}})
\end{align}
where the inequality follows simply because $\pi(N_A)\geq 0$.
\end{proof}

\begin{figure}[htb]
    \centering
    \begin{tikzpicture}
        \draw (0,0) ellipse (0.75 and 1.25);
        \draw (3,0) ellipse (1 and 1.5);
        \node at (-1,-0.5) () {$\Omega_n$};
        \node at (4.3,-0.5) () {$\Omega_{\ovn}$};
        \draw (-0.75,0) -- (0.75,0) (2,0) -- (4,0);
        \node at (0,-0.7) () {$S_n$};
        \node at (3,-0.7) () {$S_{\ovn}$};
        \draw [color=red,line width=1.1pt] (3,-0.5) -- (3,0.5);
        \draw [color=yellow, line width=1.1pt] (3,-0.5) -- (0,0.5) ;
        \draw[color=green, line width=1.1pt] (0,-0.5) -- (0,0.5);
        \draw [color=blue,line width=1.1pt] (0,-0.5) -- (3,0.5);
        
    \end{tikzpicture}
    \caption{A schematic view of the restriction chains.\\
    yellow, red, blue, and green edges correspond to $Q(S_n,\Omega_n\setminus S_n), Q(S_{\ovn},\Omega_{\ovn}\setminus S_{\ovn}),Q(S_n,\Omega_{\ovn}\setminus \Omega_{\ovn}\setminus S_{\ovn}),$ and $Q(S_{\ovn},\Omega_{\ovn}\setminus S_{\ovn})$, respectively}
    \label{fig:my_label}
\end{figure}

\paragraph{High level idea of the proof of \autoref{thm:mainthmdiscretechain}. }
We follow a proof strategy similar to  \cite{mihail1992expansion}, 
which obtains  analogue of our result in an unweighted setting 
and for the Metropolis-Hastings samplers. We use an inductive 
argument to prove the theorem. We need to prove 
$Q(S,\overline{S})\geq \frac{\pi(S)}{Ck^2}$ for  a subset 
$S \in \Omega$ with $\pi(S) \leq \frac1{2}$. Letting $S_n=S\cap \Omega_n$ and $S_{\ovn}=S \cap \Omega_{\ovn}$, we have
\begin{equation}
\label{eq:allterms}
Q(S,\overline{S}) = Q(S_n,\Omega_n\setminus S_n)+ 
Q(S_{\ovn},\Omega_{\ovn}\setminus S_{\ovn}) + 
Q(S_n,\Omega_{\ovn}\setminus 
S_{\ovn})+Q(S_{\ovn},\Omega_n\setminus S_n).
\end{equation} We carry out the induction step by lowerbounding 
the RHS of the above term by term. In order to bound  
$Q(S_n,\Omega_n\setminus S_n)$ we use induction hypothesis on $\M_n$. To bound $Q(S_{\ovn}, \Omega_n\setminus S_{\ovn})$, we combine the induction hypothesis on $\M_{\ovn}$ with \autoref{lem:induconbarn}.
 It remains to bound the other two terms which correspond to the contribution of the edge across $(\Omega_n,\Omega_{\ovn})$. 
To do that, we  crucially use 
negative association of $\pi$. 
In particular, we use the 
following lemma (appeared 
before in  \cite{mihail1992expansion} in the 
unweighted case).  For any set $A \in \Omega_n$, let \mbox{$N_{\ovn}(A)=\{\bm{y} \in \Omega_{\ovn}: \, \, \exists \bm{x}\in A,\, \, P(\bm{x},\bm{y})>0\}$} denote the set of neighbors of $A$ in $\Omega_{\ovn}$. 
\begin{lemma}[\cite{anari2016monte}]
\label{lem:expansionration}
For any subset $A \subseteq \Omega_n$, 
$$\pi_{\ovn}(N_{\ovn}(A)) \geq \pi_n(A).$$
\end{lemma}
\noindent The lemma lower bounds the vertex expansion of $S_n$ in $\Omega_n$ and similarly 
vertex expansion of $S_{\ovn}$ in $\Omega_n$. Later we show how to use it to bound the edge expansion which is 
our quantity of interest.\\

\begin{proofof}{\autoref{thm:mainthmdiscretechain}}
We induct on $k+n$. So, assume,  the conductance of the 
Gibbs sampler for any $(k-1)$-DPP over $n-1$ elements is at most $\frac{1}{C(k-1)^2}$ and the conductance is at most 
$\frac{1}{Ck^2}$ for any $k$-DPP over any $n-1$ elements.

Fix a set $S\subset \Omega$  
where $\pi(S) \leq \frac1{2}$. We need to show $Q(S,\overline{S}) \geq \frac{\pi(S)}{Ck^2}$. 
First, consider a simple case where  $\pi_n(S) \leq 
\frac{1}{2}$ and $\pi_{\ovn}(S) \leq \frac{1}{2}$. By induction hypothesis we have $Q_n(S_n,\Omega_n\setminus S_n) \geq 
\frac{\pi_n(S_n)}{c(k-1)^2}$. Moreover, by adding up  
\eqref{eq:edgeinchainn} for the edges across the cut $(S_n,\Omega_n\setminus S_n)$, we get \mbox{$Q(S_n,\Omega_n\setminus S_n) 
=\frac{(k-1)\pi(\Omega_n)}{k}\cdot Q_n(S_n,\Omega_n\setminus S_n)$}. So combining them we have 
\begin{equation}
\label{eq:inductionk}
Q(S_n,\Omega_n\setminus S_n)\geq \frac{\pi(S_n)}{Ck^2}
.\end{equation}
Now, we use induction on $\M_{\overline{n}}$ along with 
\autoref{lem:induconbarn}. The induction hypothesis implies $$Q_{\ovn}(S_{\ovn}, \Omega_{\ovn}\setminus S_{\ovn})\geq 
\frac{\pi_{\ovn}(S_{\ovn})}{ck^2}=\frac{\pi(S_{\ovn})}{\pi(\Omega_{\ovn})\cdot ck^2}$$ So to prove the 
theorem in this case, it is enough to show the following and add it up with \eqref{eq:inductionk}. 
\begin{equation}
\label{eq:weakinductionn}
    Q(S_{\ovn},\Omega_{\ovn}\setminus S_{\ovn}) +Q(S_{\ovn}, \Omega_n\setminus S_n)+ Q(S_{n}, \Omega_{\ovn}\setminus S_{\ovn})\geq \pi(\Omega_{\ovn}) \cdot Q_{\ovn}(S_{\ovn},\Omega_{\ovn}\setminus S_{\ovn}).
\end{equation}
To see that, it is enough to apply 
\autoref{lem:induconbarn} and add up \eqref{eq:auxinduct} 
for all $\bm{x} \in \Omega_n$, where subset $A \subset 
\Omega_{\overline{n}}$ in the lemma is determined as 
follows: if $\bm{x} \in S_n$ then set $A=S_{\overline{n}}$, otherwise set $A=\Omega_{\ovn}\setminus S_n$. Note that,
doing that the RHS of the result will be exactly 
$\pi(\Omega_{\ovn})\cdot 
Q_{\ovn}(S_{\ovn},\Omega_{\ovn}\setminus \Omega_{\ovn})$, 
because any edge $\bm{y}\bm{z}$  of that will only show up 
in \eqref{eq:auxinduct} by having $\bm{x}=\bm{y}\cap 
\bm{z}+n$.

So we focus on the case $\max \{ \pi_n(S_n), \pi_{\ovn}(S_{\ovn})\} > \frac1{2}$. Since $\pi(S) \leq \frac{1}{2}$, we have \mbox{$\min \{\pi_n(S_n), \pi_{\ovn}(S_{\ovn})\} \leq \frac1{2}$.}so
So, without loss of generality, perhaps by considering $\overline{S}$ instead of $S$, we may assume $\pi_n(S_n) >\frac{1}{2}$ and $\pi_{\ovn}(S) \leq \frac{1}{2}$.
  Our goal is to prove  
\begin{equation}
\label{eq:claimstatement}
Q(S,\overline{S}) \geq \frac{1}{Ck^2}\cdot \min \{1- \pi(S),\pi(S) \}
\end{equation}
For every $\bm{x}\in \Omega_n$, let $N_{\ovn,S}(\bm{x}):=N_{\ovn}(\bm{x})\cap S_{\ovn}$, and  $N_{\ovn,\overline{S}}(\bm{x}):=N_{\ovn}(\bm{x})\cap (\Omega_{\ovn}\setminus S_{\ovn})$ be a partitioning of $N_{\ovn}(\bm{x})$, so 
for every subset $T \in N_{\ovn}(\bm{x})$ we have 
\begin{equation}
    \label{eq:singleedgeQ}
    Q(\bm{x},T)= \frac{1}{2k}\cdot \frac{\pi(\bm{x})\pi(T)}{\pi(\bm{x})+\pi(N_{\ovn,S}(\bm{x}))+\pi(N_{\ovn,\overline{S}}(\bm{x}))}
\end{equation}
Now, define   $S_{\text{leave}}\subset S_n$ to be  
\begin{equation*}
S_{\text{leave}} = \{ \bm{x} \in S_n: \, \, \pi(\bm{x})+ \pi(N_{\ovn,S}(\bm{x})) < \pi(N_{\ovn,\overline{S}}(\bm{x})) \}
,\end{equation*}
in other words, $S_{\text{leave}}\in S_n$ is the subset of states so that,  if the chain takes one step from $S_{\text{leave}}$ by  removing and resampling element $n$, then with probability at least $\frac{1}{2}$ it leaves $S$ and enters $N_{\ovn,\overline{S}}(\bm{x})$. We also let $S_{\text{stay}}=S_n \setminus S_{\text{leave}}$.  On the other hand, starting from $S_{\text{stay}}$ and by  resampling $n$, the chain with probability at least half stays in $S$.
%
It is straight-forward to see    
\begin{align}
\label{eq:boundons1}
Q(S_{\text{leave}},\Omega_{\ovn}\setminus S_{\ovn}) &\geq \frac{\pi(S_{\text{leave}})}{4k} 
\end{align}
To see that, note that definition of $S_{\text{leave}}$ and 
setting $T=\Omega_{\ovn}\setminus S_{\ovn}$ in 
\eqref{eq:singleedgeQ} implies that for any $\bm{x} \in 
S_{\text{leave}}$, we have $Q(\bm{x},\Omega_{\ovn}\setminus 
S_{\ovn})\geq \frac{\pi(\bm{x})}{4k}$. To get 
\eqref{eq:boundons1}, it suffices to sum up this over all states 
of $S_{\text{leave}}$.  The bound \eqref{eq:boundons1} shows that $Q(S_{\text{leave}},\overline{S}) \gg 
\frac{\pi(S_{\text{leave}})}{k^2}$. So roughly speaking, to prove the theorem, it suffices to show $\phi(S_{\text{stay}}) \cup 
S_{\ovn}) \geq \frac{1}{Ck^2}$. 
consider two cases: if  $\pi_n(S_{\text{stay}}) \lessapprox \frac1{2}$, we essentially use the same argument as in the case $\pi_n(S_n),\pi_{\ovn}(S_{\ovn})\leq \frac12$. Otherwise we combine the induction with  \autoref{lem:expansionration} to bound the expansion. 
\begin{itemize}
\item \textbf{Case 1: $\pi_n(S_{\text{stay}}) \leq  
\frac1{2}+\frac{1}{4k}$.}
We show $Q(S,\overline{S}) \geq \frac{\pi(S)}{Ck^2}$. To do that,  we use the induction hypothesis on $\M_n$, and the following claim which is the stronger version of \eqref{eq:weakinductionn}.
\begin{claim}
\label{lem:inductiononn}
\begin{equation}
\label{eq:lemindonn}
Q(S_{\ovn}, \overline{S}) + Q(S_n, \Omega_{\ovn}\setminus{S_{\ovn}}) - \frac{1}{2}Q(S_{\text{leave}},\Omega_{\ovn}\setminus S_{\ovn})\geq \pi(\Omega_{\ovn})\cdot Q_{\ovn}(S_{\ovn},\Omega_{\ovn}\setminus S_{\ovn})
\end{equation}
\end{claim}
\begin{proof}
The claim is implied by combining  the  summation of \eqref{eq:inductiononn1},\eqref{eq:inductionebar2}, and  \eqref{eq:inductionebar3} over $\Omega_n \setminus S_n$, $S_{\text{stay}}$ and $S_{\text{leave}}$, respectively.
Let $\bm{x}  \in \Omega_n\setminus{S_n}$. Then by applying \autoref{lem:induconbarn} for $\bm{x}$ and $A=S_{\ovn}$, we get
\begin{equation}
\label{eq:inductiononn1}
\begin{aligned}
Q(N_{\ovn,S}(\bm{x}),\{\bm{x}\} \cup N_{\ovn,\overline{S}}(\bm{x}))\geq  \pi(\Omega_{\ovn})\cdot Q_{\ovn}(N_{\ovn,S}(\bm{x}),N_{\ovn,\overline{S}}(\bm{x}))
\end{aligned}
\end{equation}
Similarly if $\bm{x} \in S_n$, by applying \autoref{lem:induconbarn} for $\bm{x}$ and $A=\Omega_{\ovn}\setminus S_{\ovn}$, we have
\begin{equation}
\label{eq:inductionebar2}
Q(\bm{x} \cup N_{\ovn,S}(\bm{x}, N_{\ovn,\overline{S}}(\bm{x}))
\geq \pi(\Omega_{\ovn})\cdot Q_{\ovn}(N_{\ovn,S}(\bm{x}),N_{\ovn,\overline{S}}(\bm{x}))
.
\end{equation}

Finally, for $\bm{x} \in S_{\text{leave}}$, we have
\begin{equation}
\begin{aligned}
\label{eq:inductionebar3}
    Q(N_{\ovn,S}(\bm{x}), N_{\ovn,\overline{S}}(\bm{x}))+\frac{1}{2} Q(\bm{x},N_{\ovn,\overline{S}}(\bm{x})) &= 
    \frac{\pi(N_{\ovn,\overline{S}}(\bm{x}))}{2k\cdot (\pi(\bm{x})+\pi(N_{\ovn,S}(\bm{x}))+\pi(N_{\ovn,\overline{S}}(\bm{x})))}\cdot\left( \pi(N_{\ovn,S}(\bm{x}))+ \frac{\pi(\bm{x})}{2}\right)&\\
    &\geq \frac{1}{2k}\cdot \frac{\pi(N_{\ovn,S}(\bm{x}))\pi(N_{\ovn,\overline{S}}(\bm{x}))}{\pi(N_{\ovn,S}(\bm{x}))+\pi(N_{\ovn,\overline{S}}(\bm{x}))} &
    \\&=\pi(\Omega_{\ovn})\cdot Q_{\ovn}(N_{\ovn,S}(\bm{x}),N_{\ovn,\overline{S}}(\bm{x}))&,
\end{aligned}
\end{equation}
where the inequality follows since $\pi(\bm{x})+\pi(N_{\ovn,S}(\bm{x})) < \pi(N_{\ovn,\overline{S}}(\bm{x}))$ for $\bm{x}\in S_{\text{leave}}$.
\end{proof}
In particular, we use the above claim to get 
\begin{equation}
\begin{aligned}
\label{eq:firstboundonexp1}
Q(S,\overline{S})&= Q(S_n,\Omega_n\setminus S_n)+ 
Q(S_{\ovn},\Omega_{\ovn}\setminus S_{\ovn}) + 
Q(S_n,\Omega_{\ovn}\setminus 
S_{\ovn})+Q(S_{\ovn},\Omega_n\setminus S_n)\\
&\geq Q(S_n,\Omega_n \setminus S_n)+\frac1{2}Q(S_{\text{leave}}+\Omega_{\ovn}\setminus S_{\ovn})+\pi(\Omega_{\ovn})Q_{\ovn}
(S_{\ovn},\overline{S_{\ovn}}) \hspace{3mm}\text{By \autoref{lem:inductiononn}}\\ 
&\geq \frac{\pi(\Omega_n)-\pi(S_n)}{Ck(k-1)}+\frac1{2}Q(S_{\text{leave}},\Omega_{\ovn}\setminus S_{\ovn}) +\frac{\pi(S_{\ovn})}{Ck^2} \hspace{3mm}\text{induction Hyp. on $\M_n$ and $\M_{\ovn}$}\\
&\geq \frac{\pi(\Omega_n)-\pi(S_{\text{leave}})-\pi(S_{\text{stay}})}{Ck(k-1)}+ \frac{\pi(S_{\text{leave}})}{8k}+ +\frac{\pi(S_{\ovn})}{Ck^2} \hspace{3mm} \text{ By \eqref{eq:boundons1} and $S_n=S_{\text{leave}}\cup S_{\text{stay}}$}
\end{aligned}
\end{equation}\\
To finish the proof, we need to show the RHS of the above is at least $\frac{\pi(S)}{Ck^2}$. To see that note that since 
$ \frac{\pi(S_{\text{leave}})}{8k} \geq \pi(S_{\text{leave}})\cdot \left(\frac1{Ck^2}+\frac1{Ck(k-1)}\right)$ for sufficiently large $k$, it suffices to show
$\frac{\pi(\Omega_n)-\pi(S_{\text{stay}})}{Ck(k-1)} \geq \frac{\pi(S_{\text{stay}})}{Ck^2},$
which can be directly verified for $\pi_n(S_{\text{stay}}) \leq \frac{1}{2}+\frac{1}{4k}$.
\item \textbf{Case 2: $\pi_n(S_{\text{stay}}) > \frac1{2}+\frac{1}{4k}$.}
We prove 
$$Q(S,\bar{S}) \geq \frac{1-\pi(S)}{Ck^2}.$$ 
\autoref{lem:expansionration} states that the vertex expansion of $S_{\text{stay}}$  is proportional to $\pi_n(S_{\text{stay}})-\pi_{\ovn}(S_{\ovn})$( which is positive in this case by the assumption). We use it to bound $Q(S,\overline{S})$ by relating   vertex expansion of $S_{\text{stay}}$ to  $Q(S,\overline{S})$. In particular, we show the following claim.
\begin{claim}
\label{lem:secondbound}
\begin{equation*}
Q(S_{\text{stay}},\Omega_{\ovn}\setminus S_{\ovn}) + Q(S_{\ovn},\Omega_{\ovn}\setminus S_{\ovn}) \geq \frac{\pi(\Omega_{\ovn})}{2k}\cdot \left(\pi_n(S_{\text{stay}})-\pi_{\ovn}(S_{\ovn})\right)  
\end{equation*}
\end{claim}
\begin{proof}
Note that for any $\bm{x} \in S_{\text{stay}}$, since $\pi(N_{\ovn,\overline{S}}(\bm{x})) \leq \pi(\bm{x})+\pi(N_{\ovn,S}(\bm{x}))$, we have
$$Q(\bm{x},N_{\ovn,\overline{S}}(\bm{x}))+Q(N_{\ovn,S}(\bm{x}),N_{\ovn,\overline{S}}(\bm{x}))= \frac{1}{2k} \cdot 
\frac{\pi(N_{\ovn,\overline{S}}(\bm{x}))\cdot 
[\pi(\bm{x})+\pi(N_{\ovn,S}(\bm{x}))]}{\pi(\bm{x})+
\pi(N_{\ovn,S}(\bm{x}))+\pi(N_{\ovn,\overline{S}}(\bm{x
}))} \geq \frac{1}{2k} \cdot 
\frac{\pi(N_{\ovn,\overline{S}}(\bm{x}))}{2},$$
To complete the proof, it is enough to sum up the above over  $ S_{\text{stay}}$ to get the following  \begin{align*}
Q(S_{\text{stay}},\Omega_{\ovn}\setminus S_{\ovn}) + Q(S_{\ovn},\Omega_{\ovn}\setminus S_{\ovn}) &\geq \sum_{x \in S_{\text{stay}}} \frac{\pi(N_{\ovn,\overline{S}}(\bm{x}))}{4k} \geq \pi\left( \bigcup_{\bm{x} \in S_{\text{stay}}} N_{\ovn,\overline{S}}(\bm{x}) \right).\\
&\geq \pi(N_{\ovn}(S_{\text{stay}})) -\pi(S_{\ovn})\\ 
&\geq \pi(\Omega_{\ovn})\cdot \left(\pi_n(S_{\text{stay}})-\pi_n(S_{\ovn})\right)  \hspace{2mm} \text{By \autoref{lem:expansionration}}
\end{align*}
\end{proof}
\autoref{lem:secondbound} and \eqref{eq:firstboundonexp1} implies $Q(S,\overline{S})\geq \max \{L_1,L_2\}$ defined as above 
\begin{align*}
L_1 &:=  \frac{\pi(S_1)}{8k}+ \frac{\pi(\Omega_n)-\pi(S_{\text{leave}})-\pi(S_{\text{stay}})}{Ck(k-1)}+\frac{\pi(S_{\ovn})}{Ck^2}&  \text{ By \eqref{eq:firstboundonexp1}} \\
L_2&:= \frac{\pi(\Omega_{\ovn})}{4k}\cdot \left(\pi_n(S_{\text{stay}})-\pi_{\ovn}(S_{\ovn})\right)&\text{By \autoref{lem:secondbound}}.
\end{align*}
So we need  to prove $\max\{L_1,L_2\} \geq \frac{1-\pi(S)}{Ck^2}$. To prove that, we show that
$L_1+\frac{L_2}{k-1} \geq (1+\frac{1}{k-1})\cdot \frac{1-\pi(S)}{Ck^2}$. 
Replacing values of $L_1$ and $L_2$ in the above and simplifying the resulting inequality, we need to show
$$ \frac{\pi(S_{\text{leave}})}{8k}+\frac{\pi(S_{\ovn})}{Ck^2}+\frac{\pi(\Omega_{\ovn})}{4k(k-1)}\cdot \left(\pi_n(S_{\text{stay}})-\pi_{\ovn}(S_{\ovn})\right)\geq \frac{\pi(\Omega_{\ovn})-\pi(S_{\ovn})}{Ck(k-1)}.$$
Ignoring the $\frac{\pi(S_1)}{8k}$ term and rearranging the other terms, it is enough to show
\begin{align*}
\frac{\pi(\Omega_{\ovn})}{4k(k-1)}\cdot (\pi_{\ovn}(S_{\text{stay}})-\pi_{\ovn}(S_{\ovn})) \geq \frac{\pi(\Omega_{\ovn})}{Ck(k-1)}\cdot(1-\frac{2k-1}{k}\cdot \pi_{\ovn}(S_{\ovn})).
\end{align*}
The above  can be verified for $C>16$, by noting that by assumption $\pi_n(S_{\text{stay}}) \geq \frac1{2}+\frac1{4k}$ and $\pi_{\ovn}(S_{\ovn}) \leq \frac1{2}$.
\end{itemize}
\end{proofof}

\section{Gibbs Sampling for Continuous $k$-DPP}
\label{sec:continuousDPP}
In this section we analyze the mixing time of Gibbs samplers for continuous $k$-DPPs.
Let  $\M$ be the Gibbs sampler for a $k$-DPP defined by a  continuous kernel $L$. In \autoref{sub:conductance}, we show $\phi(M)\gtrsim \frac{1}{k^2}$. Therefore,  Gibbs
sampling is an efficient method  to generate samples
from a continuous $k$-DPP provided that: We have access to 
an \srsc{1}{k-1} oracle of $L$ to simulate the chain, and we can find a \emph{proper} starting 
distribution. In \autoref{subsec:StartingDist}, we show access to conditional oracles sampling is also enough to find the proper starting distributions. 

As alluded to before, throughout the 
section $L:\R^d \times \R^d\to \R$ is 
a continuous kernel which satisfies 
the Mercer's condition and also $\int 
\int |L(x,y)|^2 dx dy <\infty$ which 
also implies the partition function 
$Z=\int \dots \int 
\det_L(x_1,\dots,x_k)dx_k\dots dx_1 < 
\infty$. 
\subsection{Conductance of $\M$}
\label{sub:conductance}
\begin{theorem}
\label{thm:ergodicflow}
Let $\M$ be the Gibbs sampler  for a $k$-DPP defined by kernel $L$, then
\begin{equation}
\label{eq:cheegerforDPP}
    \phi(\M) \gtrsim \frac{1}{k^2}.
\end{equation}
\end{theorem}
\begin{proof}
Recall that by 
\autoref{thm:mainthmdiscretechain} the conductance of a Gibbs sampler for
any discrete $k$-DPP is at least
$\Omega(\frac{1}{k^2})$. The key 
observation is that this bound is 
independent of the number of states.
Therefore, we can obtain this bound 
for  arbitrarily fine 
discretizations of $\M$, and with a 
limiting argument extend it to $\M$.

For simplicity,  we assume $d=1$. It is 
straight-forward to extend the 
argument to higher dimensions. Let us denote the state space by $\Omega$.
Fix a measurable subset $S \subset \Omega$ with $\pi(S) \leq \frac1{2}$. Our goal is to prove
    $\phi(S)= \frac{Q(S,\overline{S})}{\pi(S)} \geq \Omega(\frac{1}{k^2}).$
Without loss of generality, we can 
only consider restriction of $\Omega$  and $S$ to a bounded set. To see 
that, note that if we set $\Omega_n = 
\binom{[-n,n]}{k}$, then clearly, 
$\lim_{n \to \infty} \frac{Q(S 
\cap\Omega_n,\overline{S}\cap \Omega_n)}{\pi(S \cap \Omega_n)} = \phi(S)$, and so for large 
values of $n$, $\frac{Q(S 
\cap\Omega_n,\overline{S})}{\pi(S \cap \Omega_n)} = \Theta(\phi(S))$.
So suppose that $\Omega = \binom{[0,1]}{k}$.    
For an integer $n$, we consider  a 
discretization $\M_n$ of $\M$ defined as 
follows. We use $n$ in subscript to denote 
quantities related to $\M_n$. 
We partition $[0,1]$ into intervals of length 
$\frac{1}{n}$, and identify each interval with an element in the ground set of $\M_n$, so $\Omega_n=\binom{[n]}{k}$. 
$\M_n$ is defined by a kernel $L_n$ 
characterized below. For $i \in [n]$ let 
$I_i=[\frac{i-1}{n},\frac{i}{n}]$. For any $i,j \in [n]$, we define  $L_n(i,j)  =  
\int\limits_{I_i}\int\limits_{I_j}
L(u,v)dudv,$
be the accumulative value of $L$ over 
$I_{i}\times I_{j}$. One can easily see $L_n$ is 
a PSD matrix, as $L$ is a PSD 
operator. Moreover, $L$ and 
consequently $\det_L$ is a continuous 
function on a closed domain, so it is 
uniform continuous, implying for any $\epsilon >0$, 
there exists an integer  $n(\epsilon)$ so that for all 
$n>n(\epsilon)$ and any two states $\{x_1,\dots,x_k\}$ 
and $\{y_1,\dots,y_k\}$ with $|y_i - x_i| \leq \frac{1}{n}$, we have $|\det_L(x_1,\dots,x_k)-\det_L(y_1,\dots,y_k) |\leq \epsilon$. 
Now, note that $f_\pi(y_1,\dots,y_k)= \frac{\det_L(y_1,\dots,y_k)}{\frac1{k!}\int \det_L(x_1,\dots,x_k) dx_1\dots dx_k}$.
So, using the simple fact that for any two sequences of numbers $\{a_n\}$ and $\{b_n\}$,
\begin{equation}
\label{eq:fraclimit}
    \left(\lim_{n \to \infty} a_n = a\right) \wedge \left(\lim_{n \to \infty } b_n = b \neq 0 \right) \implies \lim_{n \to \infty} \frac{a_n}{b_n}=\frac{a}{b}
\end{equation}
we get  that for any $\epsilon>0$, there exists an integer $m(\epsilon)$, where $m(\epsilon)$ depends on $n(\epsilon)$, such that 
\begin{equation}
\label{eq:contmeasure}
    \forall n \geq m(\epsilon), \forall  \{t_1,\dots,t_k\} \in \binom{[n]}{k}: \hspace{1mm} \left| \pi_n(t_1,\dots,t_k)-\pi(\prod_{i=1}^k I_{t_i}) \right| \leq \frac{\epsilon}{n^k}
\end{equation}
We define a set  $S_n \subset \Omega_n$ corresponding to 
$S$ for any $n$, so that
\begin{equation}\label{eq:goalcontodisc}
\lim_{n \to \infty}\phi_{n}(S_n)=\phi(S).
\end{equation}
Clearly, the above proves the theorem as by \autoref{thm:mainthmdiscretechain}, we know that $\phi_n(S_n) \gtrsim \frac{1}{k^2}$ for any $n$.
In what follows, we use $A \subset 
B$ to denote both of  $A-B$ and $B-A$ have 
Lebesgue measure zero. Also, 
 define 
 $$S_n = \left\{ \{t_1,\dots,t_k\} \in 
\binom{[n]}{k}\, \middle| \hspace{1mm} I_{t_1}\times \dots \times I_{t_k} 
\subset S\right\}.$$
Following \eqref{eq:fraclimit}, to prove \eqref{eq:goalcontodisc},  it is enough to argue that $\lim_{n \to \infty}
Q_n(S_n,\overline{S_n})= 
Q(S,\overline{S})$, and $\lim_{n \to 
\infty} \pi_n(S_n)=\pi(S)$. We first show the latter. This follows by \eqref{eq:contmeasure} and that 
\begin{equation}\label{eq:limsN}\lim_{n \to \infty} \mu\left(\cup_{\{t_1,\dots,t_k\} \in S_n} \prod_{i=1}^k I_{t_i}\right)=\mu(S)	
\end{equation}
 for $\mu$ being the Lebesgue measure.

It remains to see $\lim_{n \to \infty}
Q_n(S_n,\overline{S_n})= 
Q(S,\overline{S})$. First, note that $[0,1]^{k-1}$ is a closed set, so for any 
$\delta > 0$ and $\epsilon >0$, there exists 
an integer $n(\delta,\epsilon)$ so that for 
any $n >n(\delta,\epsilon)$, and points 
$x_1,\dots,x_k,x_{k+1}$ and 
$y_1,\dots,y_k,y_{k+1}$ with $|x_i-y_i| \leq 
\frac{1}{n}$, and  $\int_{0}^1 
\det_L(x_1,\dots,x_{k-1},\tau) d\tau \geq 
\delta$, we have 
$$\left|\frac{\det_L(x_1,\dots,x_k)\det_L(x_1
,\dots,x_{k-1},x_{k+1})}{\int_{0}^1 
\det_L(x_1,\dots,x_{k-1},\tau) d \tau}- 
\frac{\det_L(y_1,\dots,y_k)\det_L(y_1,\dots,y
_{k-1},y_{k+1})}{\int_{0}^1 
\det_L(y_1,\dots,y_{k-1},\tau) d \tau}\right| \leq \epsilon.$$ 
Therefore, similar to the case for $\pi_n$,  it 
follows that for any $\epsilon,\delta >0$, there 
exists integer $m(\delta,\epsilon)$ depending on 
$n(\delta,\epsilon)$ so that for any $n \geq 
m(\delta,\epsilon)$ and for all $t_1,\dots,t_{k-1},s,t \in {[n]\choose k+1}$   with  $\sum_{i=1}^ 
n\pi_n(t_1,\dots,i) \geq \frac{\delta}{n^{k-1}}\,$
\begin{equation}
\label{eq:limQ} \left|Q_n(\{t_1,\dots,t_{k-1},t\},\{t_1,\dots,t_{k-1},s\})-Q(I_{t}\times \prod_{i=1}^{k-1} I_{t_i}, I_{s}\times \prod_{i=1}^{k-1}  I_{t_i}) \right| \leq \frac{\epsilon}{n^{k+1}}. 
\end{equation}
Now, combining the above equation with 
\eqref{eq:limsN}, and noting $\epsilon$ and $\delta$ can be chosen arbitrary close to zero, we obtain
$\lim_{n \to \infty} Q_n(S_n,\overline{S_n}) = Q(S,\overline{S})$, which completes the proof.
\end{proof}
Combining the theorem with
\autoref{thm:cheegerinequality}, we get that $\lambda_{\M}\gtrsim \frac{1}{k^4}$, where $\lambda_{\M}$ is the poincare constant of $\M$.
Moreover, clearly the above argument 
implies the chain is $\pi$-strongly irreducible as well. So  we can apply  \autoref{thm:markovchainmixing} to obtain the following corollary.
\begin{corollary}
\label{thm:mainmixing}
Let $\pi$ be the $k$-DPP defined by $L$. If $\mu$ is an
arbitrary starting distribution, then
\begin{equation*}
    \tau_\mu(\epsilon) \leq O(k^4)\cdot\log\left(\frac{\Var_\pi(\frac{f_\mu}{f_\pi})}{\epsilon }\right).
\end{equation*}
\end{corollary}


\subsection{Finding a Starting }\label{subsec:StartingDist}
In this subsection, we prove the 
following theorem, which shows that if we have access 
to \srsc{i}{1} oracles of the kernel for any $0 \leq i \leq k-1$, then  a proper starting distribution for the associated Gibbs sampler can be found.
\begin{theorem}
\label{thm:main1}
Let $\M$ be the Gibbs sampler for the $k$-DPP defined by kernel $L:\R^d\times \R^d \to \R$. There is a polynomial
time algorithm which given access to
\frsc{i}{1}{L} oracles  all $0\leq i \leq k-1$, returns a state of $\M$
from a distribution $\mu$ where
\begin{equation}
    \label{eq:goalstart}
    \tau_\mu(\epsilon) \leq O(k^5 \log \frac{k}{\epsilon} ).
\end{equation}
Moreover, the algorithm only uses $k$ oracle accesses.
\end{theorem}

To prove the above theorem, and 
generate a sample from such a 
distribution $\mu$, we use Algorithm 
\ref{alg:startingdist} which is the 
continuous analog of a greedy 
algorithm analyzed at 
\cite{deshpande2007sampling} as approximate volume sampling. In 
particular, we crucially use the 
following lemma which directly follows \cite{deshpande2007sampling}. As always, $\pi$ denotes our $k$-DPP. 
\begin{lemma}
\label{lem:pickcenter}
Let $\nu$ be the probability distribution of the output of Algorithm \ref{alg:startingdist}. Then,  for any $\{x_1,\dots,x_k\} \subset \R^d$,
$$f_\nu(\{x_1,\dots,x_k\}) \leq  (k!)^2 f_\pi(\{x_1,\dots,x_k\}).$$
\end{lemma}
We include the proof in the appendix for the sake of completeness.
\begin{algorithm}
\begin{algorithmic}[1]
	\Input A kernel $L$ and \frsc{i}{1}{L} oracles  for $0\leq i\leq k-1$.
	\State Let $\bm{x}=\{\}$.
	\For{ $i$ from $0$ to $k-1$}
	    \State Use the \frsc{i}{1}{L} oracle to generate a sample $x_i$ and add $x_i$ to $\bm{x}$.
	\EndFor
	\Return $\bm{x}$
\end{algorithmic}
\caption{Choosing a starting state for the Gibbs sampler}
\label{alg:startingdist}
\end{algorithm}

\begin{proofof}{\autoref{thm:main1}}
First of all, clearly the algorithm 
use each \frsc{i}{1}{L} oracle for $1 \leq i \leq k-1$ once. So letting 
$\mu$ be the distribution of the 
output of the algorithm, it suffices to show  \eqref{eq:goalstart}.  Applying \autoref{thm:mainmixing}, it is equivalent to show $\Var_\pi(\frac{f_\mu}{f_\pi}) \leq O(k\log k)$. It straight-forwardly follows by applying \autoref{lem:deshpaper}. More precisely,
$$\Var_\pi(\frac{f_\mu}{f_\pi}) = \EEE{\pi}{\left(\frac{f_\mu(\bm{x})}{f_\pi(\bm{x})}\right)^2}-1 \leq (k!)^4 \cdot \EEE{\pi}{1}= (k!)^4$$
which completes the proof.

\end{proofof}\vspace{2mm}
 
 \begin{remark}
It is straight-forward to use a 
similar discretization argument to 
prove \autoref{thm:ergodicflow}, and 
consequently 
\autoref{thm:main1} when the 
domain of the kernel is restricted to 
a closed subset $C \subset \R^d$ which can be nicely discretized as in
\autoref{thm:ergodicflow}. In 
particular, we assume $C$ is an sphere in the next section. More precisely,  
$C$ could be any closed 
subset which its interior has also the same measure.
\end{remark}

\section{Applications for Sampling from Gaussian $k$-DPP's}
\label{sec:app}
As pointed out before, to find a proper starting state, and simulate the Gibbs sampler for a
$k$-DPPs, we need to have access to \srsc{i}{1} ($0\leq k-1$) sampling oracles of the kernel. In this section, we study 
the problem for Gaussian kernels, and as a special case 
argue a simple \emph{rejection sampling} algorithm is an efficient \srsc{i}{1} oracle, when restricting the kernel to
the unit sphere. 
In particular, fix $\mathcal{G}_\sigma:\R^d\times \R^d
\to \R$ to denote the Gaussian kernel with covariance
matrix $\sigma\mathbb{I}$, $\mathcal{G}_\sigma (x,y)= 
\exp(\frac{-\norm{x-y}^2}{2\sigma^2})$. Also let 
$\mathbb{S}^{d-1}=\{x\in 
\R^d \,|\, \norm{x}=1\}$ denote the unit sphere.  We prove the following.
\begin{theorem}
\label{thm:oracleGaussian}
Let $\restr{\mathcal{G}_\sigma}{\mathbb{S}^
{d-1}}$ denote the restriction of 
$\mathcal{G}_\sigma$ to the unit sphere. For any integer $k$ and any $x_1,\dots,x_k \in \mathbb{S}^{d-1}$, Algorithm 
\ref{alg:rsc} returns a sample from the conditional distribution 
\srsc{\{x_1,\dots,x_k\}}{1} 
associated with 
$\restr{\mathcal{G}_\sigma}{\mathbb{S}^{d-1}}$. If $k \leq \exp(d/4)$ 
and $t$ is the smallest integer that 
$\frac{d^t}{t!}\geq 2k$, then the 
algorithm queries at most 
$e^{\frac{2}{\sigma^2}}\cdot 
\sigma^{2t}\cdot t!$  uniform samples 
from the sphere in expectation. Moreover, if $\sigma 
\lesssim \frac1{\sqrt{\log k}}$, the 
algorithm uses $O(1)$ samples in 
expectation.
\end{theorem}
Then,  combining with \autoref{thm:main1}, and assuming generating a sample from the normal distribution can be done in constant time, we get our main theorem for sampling from Gaussian $k$-DPPs. 
\begin{theorem}
\label{thm:Gaussian}
Let $d$, and $k\leq e^{d^{1-\delta}}$   for some 
$0<\delta<1$ be two integers.  There is a randomized 
algorithm that for any $\epsilon >0$ and $\sigma>0$, 
generates an $\epsilon$-approximate sample from the 
$k$-DPP defined by $\mathcal{G}_\sigma$ on 
$\mathbb{S}^{d-1}$ which runs in time $O(d k^5\log \frac{k}{\epsilon})$ if $\sigma \lesssim  \frac{1}{\sqrt{\log k}}$, and for the larger values of $\sigma$ the running time is bounded by   $$O(d\log  \frac{1}{\epsilon})\cdot k^{O(\frac{1}{\delta})}\cdot \sigma^{2t}.$$ 
where $t = \min \left\{t \in \mathbb{N} \hspace{1mm} \middle| \frac{d^t}{t!} \geq 2k\right\}$.
\end{theorem}
\begin{proof}
\autoref{thm:oracleGaussian} states that Algorithm 
\ref{alg:rsc} gives  \srsc{i}{1} oracles for any $i$. Having these oracles, we simulate the 
corresponding Gibbs sampler starting from the 
distribution $\mu$ given by \autoref{thm:main1}. 
Let $\mathcal{T}$ denote the cost of 
a single step of the chain. So, the 
running time of the algorithm is 
$\tau_\mu(\epsilon)\cdot \mathcal{T}$. By \autoref{thm:main1}, $\tau_\mu(\epsilon)\leq O(k^5 \log \frac{k}{\epsilon})$. So we only need to analyze $\mathcal{T}$.
Note that a uniform sample from the sphere can be 
generated by using $d$ samples from the normal 
distribution, so $\mathcal{T}$ is at most   a factor of $d$ of the bound of
 \autoref{thm:oracleGaussian}. 
By \autoref{thm:oracleGaussian}, if $\sigma \lesssim \frac{1}{\sqrt{\log k}},$ then $\mathcal{T}=O(d)$, so
$\tau_\mu(\epsilon)\cdot \mathcal{T} \leq 
\tilde{O}(dk^5 \log \frac{k}{\epsilon})$ which completes the proof in this case. For larger $\sigma$ we get that 
$$ \tau_\mu(\epsilon) \cdot \mathcal{T}\leq \tilde{O}(d k^5 \log \frac{k}{\epsilon})\cdot  (e^{\frac{2}{\sigma^2}}t! \sigma^{2t}).$$
To complete the proof note that $e^{\frac{2}{\sigma^2}} \leq k^{O(1)}$, and moreover the definition of $t$ and noting that $k \leq \exp(d^{1-\delta})$ implies $t! \leq k^{O(\frac1{\delta})}$. 
\end{proof}
We remark that in the above algorithm if $k = \poly(d)$, then $t=O(1)$, and so the running time is polynomial in terms of $d,k,\sigma$. Moreover, one can see that the same holds if $\sigma =O(1)$, as $t \leq \log  k$. 
\begin{algorithm}
\begin{algorithmic}[1]
	\Input A Gaussian kernel $\mathcal{G}$ restricted to $C \subset \R^d$, and $k$ points $x_1,\dots,x_k \in C$.
	\Output A sample from the \frsc{\{x_1,\dots,x_k\}}{1}{\mathcal{G}}.
	\State Draw a uniform sample $x$  from $C$. 
	\label{line:propose}
	\State Draw a uniform number $u$ from $[0,1]$.
	\State If $u \leq \frac{\det_\mathcal{G}(x_1,\dots,x_k,x)}{\det_{\mathcal{G}}(x_1,\dots,x_k)}$, accept and return $x$. Otherwise goto line \ref{line:propose}.
\end{algorithmic}
\caption{Rejection Sampling for sampling from the conditional distribution}
\label{alg:rsc}
\end{algorithm}

We conclude the section by proving \autoref{thm:oracleGaussian}. 
\subsection{Analysis of Algorithm \ref{alg:rsc} and Proof of \autoref{thm:oracleGaussian}}
\label{sub:gaussianEigs}
\paragraph{Correctness.}
One can show that for any set of points $\bm{x}=\{x_1,\dots,x_k\}$, and any  kernel 
$L:C\times C \to \R$ such that for all $z \in 
C$, $L(z,z)\leq 1$, the output has 
\frsc{\bm{x}}{1}{L} distribution. Clearly, any Gaussian kernel has
this property. To see that, let $y$ be the point uniformly
selected from $C$. The algorithm returns $y$ with probability $\frac{\det_L(\bm{x}+y)}{\det_L(\bm{x})}$, where we are using the fact that 
this number is at most $L(y,y)$, and 
$L(y,y)\leq 1$ by the assumption. Therefore, if $\phi$ denotes the  distribution of the output, $$f_\phi(y) = \frac{1}{\Vol(C)}\cdot \frac{\det_L(\bm{x}+y)}{\det_L(\bm{x})}\propto \det_L(\bm{x}+y),$$ which implies the output has the desired  distribution.

From now on, fix a kernel 
$\restr{\mathcal{G}_{\sigma}}{\mathbb{S}^{d-1}}$ to be the input kernel, and let $T$ denote the    
number of the steps (samples generated from the sphere) until the algorithm 
terminates. So we only need to analyze $\EE{T}$.  Let $\mu$ be the uniform distribution on $\mathbb{S}^{d-1}$. 
The probability that the algorithm accepts and 
outputs the sample generated in the current step  is $$\mP_{\substack{y \sim \mu \\ u \sim 
[0,1]}}\left[u \leq 
\det_{\mathcal{G}_\sigma}(\bm{x}+y)\right]=\EEE{y\sim 
\mu}{\frac{\det(\bm{x}+y)}{\det(\bm{x})}}.$$ So $T$ forms a geometric distribution and 
$\EE{T}=\frac{\det(\bm{x})}{\EEE{y \sim \mu}{\det(\bm{x}+y)}}$. The following lemma concludes the proof of \autoref{thm:oracleGaussian}.
\begin{lemma}
\label{lem:boundlargeeigsofgaussian}
For any parameter $\sigma \geq 0$, any integer $k\leq \exp(d/4)$ and any set of points
$x_1,\dots,x_k \in \mathbb{S}^{d-1}$, if we set $\mu$ to be the uniform distribution on $\mathbb{S}^{d-1}$ and  $t$ to be the smallest number such that $\frac{d^t}{t!} \geq 2k$, then \begin{equation}
\label{eq:lemeigval}
    \EEE{y\sim 
\mu}{\frac{\det_{\mathcal{G}_\sigma}(\bm{x}+y)}{\det_{\mathcal{G}_\sigma}(\bm{x})}} \gtrsim \frac{e^{\frac{-2}{\sigma^2}}}{t! \cdot \sigma^{2t}}
\end{equation}
Moreover, if $\sigma \lesssim \frac{1}{\sqrt{\log k}}$, then the bound can be improved to $\Omega(1)$.

\end{lemma}
 To prove the lemma, we relate the quantity $\EEE{y\sim 
\mu}{\frac{\det_{\mathcal{G}_\sigma}(\bm{x}+y)}{\det_{\mathcal{G}_\sigma}(\bm{x})}}$, to eigenvalues of $\restr{\mathcal{G}_\sigma}{\mathbb{S}^{d-1}}$ and  use the work of 
\cite{minh2006mercer} who studied eigenvalues, 
and eigenspaces of Gaussian kernels.
Set $\tilde{\mathcal{G}}=\frac{\restr{\mathcal{G}_\sigma}{\mathbb{S}^{d-1}}}{\Vol(\mathbb{S}^{d-1})}$ be the kernel normalized with the uniform measure. In particular we use the following 
theorem.
\begin{theorem}[\cite{minh2006mercer}]
\label{thm:eigenvaluesgaussian}
For any integer $\ell \geq 0$, $\tilde{\mathcal{G}}$ 
has an eigenvalue $\mu_\ell$ with multiplicity 
$N(d,\ell) = \frac{(2\ell+d-2)(\ell+d-3)!}{\ell!(d-2)
!}$ where
$$\mu_\ell=e^{-\frac{2}{\sigma^2}}\sigma^{d-2}I_{\ell+\frac{d}{2}-1}(\frac{2}{\sigma^2})\Gamma(\frac{d}{2}),$$
and $I$ denotes  the modified Bessel function of the first
kind, defined by $I_\nu(z)= \sum_{i=0}^\infty \frac{1}{i!(\nu+i+1)!}(\frac{z}{2})^{\nu+2i}$.
Also, for any integer $\ell$, $\mu_\ell$ satisfies the following.
\begin{equation}
\left(\frac{2e}{\sigma^2}\right)^\ell \cdot \frac{A_1}{(2\ell+d-2)^{\ell+\frac{d-1}{2}}} \leq \mu_\ell \leq 
\left(\frac{2e}{\sigma^2}\right)^\ell \cdot \frac{A_2}{(2\ell+d-2)^{\ell+\frac{d-1}{2}}},
\end{equation}
where $A_1 = e^{-\frac{2}{\sigma^2}-\frac1{12}}\frac{1}{\sqrt{\pi}}(2e)^{\frac{d}{2}-1}\Gamma\left({\frac{d}{2}}\right)$ and $A_2 = A_1\cdot e^{\frac{1}{12}+\frac{1}{\sigma^4}}$. 
\end{theorem}

\begin{proofof}{\autoref{lem:boundlargeeigsofgaussian}}
Since $\mathcal{G}_\sigma$ is a PSD operator, for any $x \in 
\mathbb{S}^{d-1}$, there exists 
function (feature map) $f_x:\mathbb{S}^{d-1} \to \R$ such that for any $y \in \mathbb{S}^{d-1}$, $\mathcal{G}_\sigma(x,y) = \langle f_x,f_y
\rangle$. 
For any  $y \in \mathbb{S}^{d-1}$, define  
$\mathcal{E}(y) 
=\Pi_{\langle f_{x_1},\dots,f_{x_k} \rangle^\perp}(f_y)$, be the projection of $f_y$ onto the space orthogonal to vectors corresponding to $x_1,\dots,x_k$. Then, by definition
$\frac{\det(\bm{x}+y)}{\det(\bm{x})}= \norm{\mathcal{E}(y)}^2$, where recall that $\bm{x}=\{x_1,\dots,x_k\}$. It implies 
\begin{equation}
\label{eq:simplepartfunc}
    \EEE{ y\sim \mu}{\frac{\det(\bm{x}+y)}}{\det(\bm{x})}=\EEE{y \sim 
\mu}{\norm{\mathcal{E}(y)}^2}=\frac{\Tr(\mathcal{
E})}{\Vol(\mathbb{S}^{d-1})}
\end{equation} for the kernel
$\mathcal{E}:\mathbb{S}^{d-1}\times 
\mathbb{S}^{d-1}\to \R$ defined by 
$\mathcal{E}(x,y)=\langle 
\mathcal{E}(x),\mathcal{E}(y) \rangle$. We further simplify this by noting that 
 $\mathcal{E}$ satisfies Mercer's condition, as $\mathcal{E}(.,.)$ is a PSD kernel. It implies 
 $\Tr(\mathcal{E})=\sum_{i=1}^{\infty} 
 \lambda_i(\mathcal{E})$. Moreover, it follows from the 
 definition of $\mathcal{E}$, that 
 $\restr{\mathcal{G}_\sigma}{\mathbb{S}^{d-1}}-\mathcal{E
 }$ is an operator of rank at most $k$. So 
 $\sum_{i=1}^{\infty} \lambda_i(\mathcal{E})\geq 
 \sum_{j=k+1}^\infty \lambda_j(\restr{\mathcal{G}_\sigma}{\mathbb{S}^{d-1}})$. 
So recalling  $\tilde{\mathcal{G}}= \restr{\mathcal{G}_\sigma}{\mathbb{S}^{d-1}}$, and using  \eqref{eq:simplepartfunc}, we get $\EEE{ y\sim \mu}{\frac{\det(\bm{x}+y)}}{\det(\bm{x})} \geq \sum_{j=k+1}^\infty \lambda_j(\tilde{\mathcal{G}}_\sigma)$. We first prove, if $\sigma \leq \frac{1}{2\sqrt{\log k}}$, then $\sum_{j=k+1}^\infty \lambda_j(\tilde{\mathcal{G}}_\sigma) \geq \Omega(1)$.
Using the Cauchy-Schwarz inequality we have
$$ k\cdot \sum_{i=1}^k \lambda_i(\tilde{\mathcal{G}})^2 \geq  
\left(\sum_{i=1}^k \lambda_i(\tilde{\mathcal{G}})\right)^2 = \left(1- \sum_{i=k+1}^{\infty}\lambda_i(\tilde{\mathcal{G}}) \right)^2.$$
We show $ \sum_{i=1}^k \lambda_i(\tilde{\mathcal{G}})^2 
\leq \frac{1}{k^2}$ which implies 
$\sum_{i=k+1}^\infty \lambda_i (\tilde{\mathcal{G}}) \geq (1-1/\sqrt{k})$ which completes the
proof. To see that, note that $\sum_{i=1}^k 
\lambda_i(\tilde{\mathcal{G}})^2=\sum_{i=1}^k 
\lambda_i(\tilde{\mathcal{G}}^2) \leq 
\Tr(\tilde{\mathcal{G}}^2)$ and 
$\Tr(\tilde{\mathcal{G}}^2) = \langle 
\tilde{\mathcal{G}},\tilde{\mathcal{G}} \rangle =\EEE{x,y \sim \mu}{e^{-\norm{x-y}^2/2\sigma^2}}$, where recall 
$\mu$ is the uniform measure on the sphere. Fix $x \in 
\mathbb{S}^{d-1}$. It follows from basic concentration inequalities for Gaussian measures that $\EEE{y \sim \mu}{e^{-\norm{x-y}^2/2\sigma^2}}\leq e^{-1/2\sigma^2}$. So $\Tr(\tilde{\mathcal{G}}^2) \leq e^{-1/2\sigma^2}$ which is at most 
$\frac{1}{k^2}$ for $\sigma^2 \leq \frac{1}{4\log k}$, 
and completes the proof.    

So from now on, we only need to prove for any $\sigma$ 
\begin{equation}
    \label{eq:goalEigval}
    \sum_{i=k+1}^\infty \lambda_i(\tilde{\mathcal{G}})
    \gtrsim \frac{e^{\frac{-2}{\sigma^2}}}{t! \cdot \sigma^{2t}}.
\end{equation}
For any integer $\ell \geq 0$, let $\mu_\ell$ be the eigenvalue of $\tilde{\mathcal{G}}$ with multiplicity $n_\ell=N(\ell,d)$ given by \autoref{thm:eigenvaluesgaussian}. It suffices to show
 $\frac{n_t\mu_t}{2} \geq  \frac{e^{\frac{-2}{\sigma^2}}}{t! \cdot \sigma^{2t}}$
 where we are using the fact that for any $\ell$, $n_\ell \geq \frac{d^t}{t!}$, and so  $n_t \geq 2k$. Now using $n_t \geq \frac{d^t}{t!}$, and the bound on $\mu_t$ by \autoref{thm:eigenvaluesgaussian}, we get 
 \begin{align*}
     n_t \mu_t &\gtrsim \frac{d^t}{t!}\cdot \frac{e^{\frac{-2}{\sigma^2}}(2e)^{t+\frac{d}{2}}\Gamma(\frac{d}{2})}{\sigma^{2t}\cdot(2t+d)^{t+\frac{d+1}{2}}}&\\
     &\gtrsim \frac{d^t}{t!}\cdot \frac{e^{\frac{-2}{\sigma^2}}(2e)^t\cdot d^{\frac{d+1}{2}}}
{\sigma^{2t}(2t+d)^{t+\frac{d+1}{2}}}& \text{ Sterling's approximation }\\
&\geq \frac{e^{\frac{-2}{\sigma^2}}(2e)^t}
{\sigma^{2t} \cdot t!\cdot  (1+\frac{2t}{d})^{t+\frac{d+1}{2}}}  \gtrsim \frac{e^{\frac{-2}{\sigma^2}}2^t}{\sigma^{2t}\cdot t!\cdot e^{\frac{2t^2}{d}}} & \text{by } (1+2t/d) \leq e^{2t/d}. 
\end{align*}
     
Noting that $k \leq \exp(d/4)$ implies $t\leq \frac{d}{4}$ and  $\exp(2t/d) \leq 2$, completes the proof of \eqref{eq:lemeigval}.

\end{proofof}

\bibliographystyle{alpha}
\bibliography{ref}
\appendix
\section{Missing Proofs}
For a vector $v$, and a linear 
subspace $H$, we use $d(v,H)$ to 
denote the distance of $v$ from $H$.
\label{missingproofs}

 \begin{proofof}{\autoref{lem:deshpaper}}
For any $x \in \R^d$, let $f_x$
be the corresponding feature map, i.e. 
$f_x: \mathcal{H}\to \R$ for some Hilbert space $\mathcal{H}$ and for any $x,y \in \R^d$, $L(x,y)= \langle 
f_x,f_y\rangle$. Fix $\bm{x}=\{x_1,\dots,x_k\}$, and let 
$S_k$ be the set of all permutations of $\{x_1,\dots,x_k\}$. Also, for any $\sigma \in S_k$
and for any $1 \leq i \leq k-1$, define $H_\sigma^i = \langle
f_{\sigma(1)},\dots,
f_{\sigma(i)}  \rangle$. In the above the range of all integrals is $\R^d$. We
have 
\begin{equation*}
f_\nu(\bm{x}) = \sum_{\sigma \in S_k} \left[ \frac{\norm{f_{\sigma(1)}}^2}{\int \norm{f_y}^2 dy}\cdot \frac{d(f_{\sigma(2)},H_\sigma^1)^2}{\int d(f_y,H_\sigma^1)^2 dx}\dots \frac{d(f_{\sigma(k)},H_\sigma^{k-1})^2}{\int d(f_y,H_\sigma^{k-1})^2 dy}\right]. 
\end{equation*}
Note that the above integrals are well-defined since our kernel is continuous.
For any $1\leq i \leq k-1$, let $H_*^i 
=\text{arg}\min_{H=\langle 
f_{y_1},\dots,f_{y_i}\rangle} \int 
d(f_y,H)^2 dy$, where $y_1\dots,y_i$ 
range over $\R^d$. Note that, the minimum 
of the quantity is defined since $L$ is continuous on a closed set. Combining with the above, and noting that for any $\sigma$, $\det(x_1,\dots,x_k)= \norm{f_{\sigma(1)}}^2\cdot d(f_{\sigma(2)},H_\sigma^1)^2\dots d(f_{\sigma(k)},H_\sigma^{k-1})^2$, we obtain
\begin{align*}
    f_\nu(\bm{x}) &\leq k!\cdot \frac{\det(x_1,\dots,x_k)}{\int \norm{f_y}^2 dy \cdot \int d(f_y,H_*^1)^2 dy\cdot \int d(f_y,H_*^{k-1}dy}\\
    &\leq k!\cdot \frac{f_\pi(\bm{x})\cdot \int \dots \int_C \det(y_1,\dots,y_k) dy_k\dots dy_1}{k!\cdot \int \norm{f_y}^2 dy\cdot \int d(f_y,H_*^1)^2 dx\dots\int d(f_y,H_*^{k-1})^2 dy}.
\end{align*}
So, rearranging the above to show $\frac{f_\nu(\bm{x})}{f_\pi(\bm{x})}
\leq (k!)^2$, it suffices to show
\begin{equation}
\label{eq:boundpinu}
  \frac{\int \dots \int \det(y_1,\dots,y_k) dy_k\dots dy_1}{\int \norm{f_y}^2 dy\cdot \int d(f_y,H_*^1)^2 dx\dots\int d(f_y,H_*^{k-1})^2 dy} \leq(k!)^2.
\end{equation}
To proof the above, we use induction on $k$. For $k=1$, the statement  is obvious as for any $y \in \R^d$, $\det(y)=L(y,y)=\norm{f_y}^2$. It is straight-forward to see, applying the above claim will prove the induction step, and completes the proof.
\begin{claim}
\label{claim:pinu}
\begin{equation}
    \int \dots \int \det(y_1,\dots,y_k) dy_k\dots dy_1 \leq k^2\left(\int d(f_y,H_*^{k-1})^2 dy
\right) \left( \int \dots \int \det(y_1,\dots,y_{k-1}) dy_{k-1}\dots dy_1 \right)
\end{equation}
\end{claim}

\begin{proofof}{ \autoref{claim:pinu}}
For any $\bm{y}=\{y_1,\dots,y_k\} \subset \R^d$, let $G_{\bm{y}}$ be a 
$(k-1)$-dimensional linear subspace of 
$\langle f_{y_1}\dots,f_{y_k} \rangle$ 
which contains the projection of 
$H_*^{(k-1)}$ onto $\langle 
f_{y_1}\dots,f_{y_k} \rangle$. Now, for any $\bm{y}$, using
\autoref{lem:deshpaper}, we get 
\begin{align*}
    \det(\bm{y}) &\leq  \left( \sum_{i=1}^k d(f_{y_i},G_{\bm{y}})\sqrt{\det(\bm{y}-y_i)} \right)^2 \\
    &\leq k\left( \sum_{i=1}^k  d(f_{y_i},G_{\bm{y}})^2\det(\bm{y}-y_i)\right) \hspace{2mm} \text{Cauchy-Schwarz Inequality}.
\end{align*}
By  integerating the above, we get
\begin{align*}
    \int \dots \int \det(\bm{y}) d\bm{y} &\leq k\int_{\R^d}  \dots \int \sum_{i=1}^k d(f_{y_i},G_{\bm{y}})^2 \det(\bm{y}-y_i) d\bm{y} \\
    &\leq k^2 \int_{y \in \R^d} \int_{z_1\in \R^d}\dots \int_{z_{k-1}\in \R^d} d(f_y,G_{\bm{z}+y})^2 \det(\bm{z})d\bm{z} dy \hspace{2mm}  \text{(setting } \bm{z}=\{z_1,\dots,z_{k-1}\}) \\
    &\leq \int_{y \in \R^d} \int_{z_1\in \R^d}\dots \int_{z_{k-1}\in \R^d} d(f_y, H_*^{k-1})^2 \det(\bm{z}) d\bm{z} dy \hspace{2mm} \\
    &= \left(\int d(f_y,H_*^{k-1})^2 dy\right) \left( \int_{z_1 \in \R^d}\dots \int_{z_{k-1} \in \R^d} \det(\bm{z}) d\bm{z} \right),
\end{align*}
where in the third inequality, the fact  $d(f_y,G_{\bm{z}+y}) \leq d(f_y,H_*^{k-1})$ holds because $f_y \in \langle f_{z_1},\dots,f_{z_{k-1}},f_y\rangle$, and $G_{\bm{z}+y}$ contains the projection of $H_{*}^{k-1}$ onto this space. Thus, the proof of the claim and the theorem is complete.
\end{proofof}
 
\end{proofof}

\begin{lemma}[Lemma 2 of \cite{deshpande2007sampling}]
\label{lem:deshpaper}
Let $S$ be a set of $k$ vectors, and $H$ be any $(k-1)$-dimesnsional  subspace of $\langle S\rangle$. Then 
$$\Vol(S) \leq \sum_{v\in S} d(v,H)\Vol(S-v),$$
where volume of a set of vectors, refer to the volume of the parallelopiped spanned by them.
\end{lemma}

\end{document}